\documentclass{article} 
\usepackage{iclr2020_conference,times}

\usepackage{amsmath,amsfonts,bm}

\def\eqref#1{equation~\ref{#1}}
\def\Eqref#1{Equation~\ref{#1}}

\def\1{\bm{1}}

\DeclareMathAlphabet{\mathsfit}{\encodingdefault}{\sfdefault}{m}{sl}
\SetMathAlphabet{\mathsfit}{bold}{\encodingdefault}{\sfdefault}{bx}{n}

\usepackage{hyperref}
\usepackage{url}
\usepackage{amsthm}
\usepackage{amssymb}
\usepackage{todonotes}
\usepackage{wrapfig}
\usepackage{graphicx}
\usepackage{subcaption}
\usepackage{multicol}

\title{Estimating Gradients for Discrete Random Variables by Sampling without Replacement}

\author{Wouter Kool \\
  University of Amsterdam \\
  ORTEC \\
  \texttt{w.w.m.kool@uva.nl} \\
  \And
  Herke van Hoof \\
  University of Amsterdam\\
  \texttt{h.c.vanhoof@uva.nl} \\
  \And
  Max Welling \\
  University of Amsterdam \\
  CIFAR \\
  \texttt{m.welling@uva.nl} \\
}

\iclrfinalcopy 
\begin{document}

\newtheorem{theorem}{Theorem}
\newtheorem{lemma}{Lemma}
\newtheorem{corollary}{Corollary}

\theoremstyle{definition}
\newtheorem{definition}{Definition}

\maketitle

\begin{abstract}
We derive an unbiased estimator for expectations over discrete random variables based on sampling \emph{without replacement}, which reduces variance as it avoids duplicate samples. We show that our estimator can be derived as the Rao-Blackwellization of three different estimators. Combining our estimator with REINFORCE, we obtain a policy gradient estimator and we reduce its variance using a built-in control variate which is obtained without additional model evaluations. The resulting estimator is closely related to other gradient estimators. Experiments with a toy problem, a categorical Variational Auto-Encoder and a structured prediction problem show that our estimator is the only estimator that is consistently among the best estimators in both high and low entropy settings.
\end{abstract}

\section{Introduction}
Put replacement in your basement! We derive the \emph{unordered set estimator}\footnote{Code available at \scriptsize \url{https://github.com/wouterkool/estimating-gradients-without-replacement}.}: an unbiased (gradient) estimator for expectations over discrete random variables based on (unordered sets of) samples \emph{without replacement}. In particular, we consider the problem of estimating (the gradient of) the expectation of $f(\bm{x})$ where $\bm{x}$ has a discrete distribution $p$ over the domain $D$, i.e.
\begin{equation}
\label{eq:expectation}
    \mathbb{E}_{\bm{x} \sim p(\bm{x})}[f(\bm{x})] = \sum\nolimits_{\bm{x} \in D} p(\bm{x}) f(\bm{x}).
\end{equation}
This expectation comes up in reinforcement learning, discrete latent variable modelling (e.g.\ for compression), structured prediction (e.g.\ for translation), hard attention and many other tasks that use models with discrete operations in their computational graphs (see e.g.\ \citet{jang2016categorical}). In general, $\bm{x}$ has structure (such as a sequence), but we can treat it as a `flat' distribution, omitting the bold notation, so $x$ has a categorical distribution over $D$ given by $p(x), x \in D$. Typically, the distribution has parameters $\bm{\theta}$, which are learnt through gradient descent. This requires estimating the gradient $\nabla_{\bm{\theta}} \mathbb{E}_{x \sim p_{\bm{\theta}}(x)}[f(x)]$, using a set of samples $S$. A gradient estimate $e(S)$ is unbiased if 
\begin{equation}
    \mathbb{E}_S[e(S)] = \nabla_{\bm{\theta}} \mathbb{E}_{x \sim p_{\bm{\theta}}(x)}[f(x)].
\end{equation}
The samples $S$ can be sampled independently or using alternatives such as stratified sampling which reduce variance to increase the speed of learning. In this paper, we derive an unbiased gradient estimator that reduces variance by avoiding duplicate samples, i.e.\ by sampling $S$ without replacement. This is challenging as samples without replacement are dependent and have marginal distributions that are different from $p(x)$. We further reduce the variance by deriving a built-in control variate, which maintains the unbiasedness and does not require additional samples.

\paragraph{Related work.}
Many algorithms for estimating gradients for discrete distributions have been proposed. A general and widely used estimator is REINFORCE 
\citep{williams1992simple}.
Biased gradients based on a continuous relaxations of the discrete distribution (known as Gumbel-Softmax or Concrete) were jointly introduced by \citet{jang2016categorical} and \citet{maddison2016concrete}. These can be combined with the straight through estimator \citep{bengio2013estimating} if the model requires discrete samples or be used to construct control variates for REINFORCE, as in REBAR \citep{tucker2017rebar} or RELAX \citep{grathwohl2017backpropagation}.
Many other methods use control variates and other techniques to reduce the variance of REINFORCE
\citep{paisley2012variational, ranganath2014black,gregor2014deep,mnih2014neural,gu2015muprop,mnih2016variational}.

Some works rely on explicit summation of the expectation, either for the marginal distribution \citep{titsias2015local} or globally summing some categories while sampling from the remainder \citep{liang2018memory,liu2019rao}. Other approaches use a finite difference approximation to the gradient \citep{lorberbom2018direct,lorberbom2019direct}. \citet{yin2019arsm} introduced ARSM, which uses multiple model evaluations where the number adapts automatically to the uncertainty.

In the structured prediction setting, there are many  algorithms for optimizing a quantity under a sequence of discrete decisions, using (weak) supervision, multiple samples (or deterministic model evaluations), or a combination both \citep{ranzato2016sequence, shen2016minimum,he2016dual,norouzi2016reward,bahdanau2017actor,edunov2018classical,leblond2018searnn,negrinho2018learning}. Most of these algorithms are biased and rely on pretraining using maximum likelihood or gradually transitioning from supervised to reinforcement learning. Using Gumbel-Softmax based approaches in a sequential setting is difficult as the bias accumulates because of mixing errors \citep{gu2017neural}. 

\section{Preliminaries}
Throughout this paper, we will denote with $B^k$ an \emph{ordered} sample without replacement of size $k$ and with $S^k$ an \emph{unordered} sample (of size $k$) from the categorical distribution $p$.

\paragraph{Restricted distribution.}
When sampling without replacement, we remove the set $C \subset D$ already sampled from the domain and we denote with $p^{D \setminus C}$ the distribution \emph{restricted to the domain $D \setminus C$}:
\begin{equation}
    p^{D \setminus C}(x) = \frac{p(x)}{1 - \sum_{c \in C} p(c)}, \quad x \in D \setminus C.
\end{equation}

\paragraph{Ordered sample without replacement $B^k$.}
Let $B^k = (b_1, ..., b_k), b_i \in D$ be an \emph{ordered sample without replacement}, which is generated from the distribution $p$ as follows: first, sample $b_1 \sim p$, then sample $b_2 \sim p^{D\setminus \{b_1\}}$, $b_3 \sim p^{D\setminus \{b_1, b_2\}}$, etc.\ i.e.\ elements are sampled one by one without replacement. Using this procedure, $B^k$ can be seen as a (partial) ranking according to the Plackett-Luce model \citep{plackett1975analysis, luce1959individual} and the probability of obtaining the vector $B^k$ is
\begin{equation}
\label{eq:p_B_k}
    p(B^k) = \prod_{i=1}^k p^{D \setminus B^{i-1}}(b_i) = \prod_{i=1}^k \frac{p(b_i)}{1 - \sum\limits_{j<i} p(b_j)}.
\end{equation}
We can also restrict $B^k$ to the domain $D \setminus C$, which means that $b_i \not \in C$ for $i = 1, ..., k$:
\begin{equation}
\label{eq:p_B_k_restr}
    p^{D \setminus C}(B^k) = \prod_{i=1}^k \frac{p^{D \setminus C}(b_i)}{1 - \sum\limits_{j<i} p^{D \setminus C}(b_j)} 
    = \prod_{i=1}^k \frac{p(b_i)}{1 - \sum\limits_{c \in C} p(c) - \sum\limits_{j<i} p(b_j)}.
\end{equation}

\paragraph{Unordered sample without replacement.}
Let $S^k \subseteq D$ be an \emph{unordered} sample without replacement from the distribution $p$, which can be generated simply by generating an ordered sample and discarding the order. We denote elements in the sample with $s \in S^k$ (so without index) and we write $\mathcal{B}(S^k)$ as the set of all $k!$ permutations (orderings) $B^k$ that correspond to (could have generated) $S^k$. It follows that the probability for sampling $S^k$ is given by:
\begin{equation}
\label{eq:P_S_def}
     p(S^k) = \hspace*{-0.3cm} \sum_{B^k \in \mathcal{B}(S^k)} \hspace*{-0.3cm} p(B^k) = \hspace*{-0.3cm} \sum_{B^k \in \mathcal{B}(S^k)} \prod_{i=1}^k \frac{p(b_i)}{1 - \sum\limits_{j<i} p(b_j)} = \left(\prod_{s \in S^k} p(s)\right) \cdot  \hspace*{-0.2cm} \sum_{B^k \in \mathcal{B}(S^k)} \prod_{i=1}^k \frac{1}{1 - \sum\limits_{j<i} p(b_j)}.
\end{equation}
The last step follows since $B^k \in \mathcal{B}(S^k)$ is an ordering of $S^k$, such that $\prod_{i=1}^k p(b_i) = \prod_{s \in S} p(s)$. Naive computation of $p(S^k)$ is $O(k!)$, but in Appendix \ref{sec:computation_P_S} we show how to compute it efficiently.

When sampling from the distribution restricted to $D \setminus C$, we sample $S^k \subseteq D \setminus C$ with probability:
\begin{equation}
    \label{eq:P_S_restr_def}
    p^{D \setminus C}(S^k) = \left(\prod_{s \in S^k} p(s)\right) \cdot \sum_{B^k \in \mathcal{B}(S^k)} \prod_{i=1}^k \frac{1}{1 - \sum\limits_{c \in C} p(c) - \sum\limits_{j<i} p(b_j)}.
\end{equation}

\paragraph{The Gumbel-Top-$k$ trick.}
As an alternative to sequential sampling, we can also sample $B^k$ and $S^k$ by taking the top $k$ of Gumbel variables \citep{yellott1977relationship,vieira2014gumbel,kim2016exact}. Following notation from \citet{kool2019stochastic}, we define the \emph{perturbed log-probability} $g_{\phi_i} = \phi_i + g_i$, where $\phi_i = \log p(i)$ and $g_i \sim \text{Gumbel}(0)$. Then let $b_1 = \arg\max_{i \in D} g_{\phi_i}$, $b_2 = \arg\max_{i \in D \setminus \{b_1\}} g_{\phi_i}$, etc., so $B^k$ is the top $k$ of the perturbed log-probabilities \emph{in decreasing order}. The probability of obtaining $B_k$ using this procedure is given by \eqref{eq:p_B_k}, so this provides an alternative sampling method which is effectively a (non-differentiable) reparameterization of sampling without replacement. For a differentiable reparameterization, see \citet{grover2018stochastic}.

It follows that taking the top $k$ perturbed log-probabilities \emph{without order}, we obtain the unordered sample set $S^k$. 
This way of sampling underlies the efficient computation of $p(S^k)$ in Appendix \ref{sec:computation_P_S}.

\section{Methodology}
In this section, we derive the 
\emph{unordered set policy gradient estimator}: a low-variance, unbiased estimator of $\nabla_{\bm{\theta}} \mathbb{E}_{p_{\bm{\theta}}(x)}[f(x)]$ based on an unordered sample without replacement $S^k$. First, we derive the generic (non-gradient) estimator for $\mathbb{E}[f(x)]$ as the Rao-Blackwellized version of a single sample Monte Carlo estimator (and two other estimators!). Then we combine this estimator with REINFORCE \citep{williams1992simple} and we show how to reduce its variance using a built-in baseline.

\subsection{Rao-Blackwellization of the single sample estimator}
\label{sec:ss_estimator}
A very crude but simple estimator for $\mathbb{E}[f(x)]$ based on the \emph{ordered} sample $B^k$ is to \emph{only} use the first element $b_1$, which by definition is a sample from the distribution $p$.
We define this estimator as the \emph{single sample estimator}, which is unbiased, since
\begin{equation}
\label{eq:ss_unbiased}
    \mathbb{E}_{B^k \sim p(B^k)}[f(b_1)] = \mathbb{E}_{b_1 \sim p(b_1)}[f(b_1)] = \mathbb{E}_{x \sim p(x)}[f(x)].
\end{equation}

Discarding all but one sample, the single sample estimator is inefficient, but we can use Rao-Blackwellization \citep{casella1996rao} to signficantly improve it.
To this end, we consider the distribution $B^k|S^k$, which is, knowing the unordered sample $S^k$, the conditional distribution over ordered samples $B^k \in \mathcal{B}(S^k)$ that could have generated $S^k$.\footnote{Note that $B^k|S^k$ is \emph{not} a Plackett-Luce distribution restricted to $S^k$!}
Using $B^k|S^k$, we rewrite $\mathbb{E}[f(b_1)]$ as
\begin{equation*}
    \mathbb{E}_{B^k \sim p(B^k)}[f(b_1)]
    = \mathbb{E}_{S^k \sim p(S^k)} \left[ \mathbb{E}_{B^k \sim p(B^k|S^k)} \left[ f(b_1) \right] \right] = \mathbb{E}_{S^k \sim p(S^k)} \left[ \mathbb{E}_{b_1 \sim p(b_1|S^k)} \left[ f(b_1) \right] \right].
\end{equation*}

The Rao-Blackwellized version of the single sample estimator computes the inner conditional expectation exactly.
Since $B^k$ is an ordering of $S^k$, we have $b_1 \in S^k$ and we can compute this as
\begin{equation}
\label{eq:E_b_1_cond_S_k}
    \mathbb{E}_{b_1 \sim p(b_1|S^k)} \left[ f(b_1) \right] = \sum_{s \in S^k} P(b_1 = s|S^k) f(s)
\end{equation}
where, in a slight abuse of notation, $P(b_1 = s|S^k)$ is the probability that the first sampled element $b_1$ takes the value $s$, given that the complete set of $k$ samples is $S^k$. Using Bayes' Theorem we find
\begin{equation}
\label{eq:P_b_1_cond_S_k}
    P(b_1 = s|S^k) = \frac{p(S^k|b_1 = s)P(b_1 = s)}{p(S^k)} = \frac{p^{D \setminus \{s\}}(S^k \setminus \{s\}) p(s)}{p(S^k)}.
\end{equation}
The step $p(S^k|b_1 = s) = p^{D \setminus \{s\}}(S^k \setminus \{s\})$ comes from analyzing sequential sampling without replacement: given that the first element sampled is $s$, the remaining elements have a distribution restricted to $D \setminus \{s\}$, so sampling $S^k$ (including $s$) given the first element $s$ is equivalent to sampling the remainder $S^k \setminus \{s\}$ from the restricted distribution, which has probability $p^{D \setminus \{s\}}(S^k \setminus \{s\})$ (see \eqref{eq:P_S_restr_def}).

\paragraph{The unordered set estimator.}
For notational convenience, we introduce the \emph{leave-one-out ratio}.
\begin{definition}
The \emph{leave-one-out ratio} of $s$ w.r.t.\ the set $S$ is given by $R(S^k, s) = \frac{p^{D \setminus \{s\}}(S^k \setminus \{s\})}{p(S^k)}$.
\end{definition}
Rewriting \eqref{eq:P_b_1_cond_S_k} as $P(b_1 = s|S^k) = p(s) R(S^k, s)$ shows that the probability of sampling $s$ first, given $S^k$, is simply the unconditional probability multiplied by the leave-one-out ratio. We now define the unordered set estimator as the Rao-Blackwellized version of the single-sample estimator.
\begin{theorem}
\label{thm:unordered_set_estimator_rao}
The \emph{unordered set estimator}, given by
\begin{equation}
\label{eq:unord_set_estimator}
    e^{\text{US}}(S^k)
    = \sum_{s \in S^k} p(s) R(S^k, s) f(s)
\end{equation}
is the Rao-Blackwellized version of the (unbiased!) single sample estimator.
\end{theorem}
\begin{proof} Using $P(b_1 = s|S^k) = p(s) R(S^k, s)$ in \eqref{eq:E_b_1_cond_S_k} we have 
\begin{equation}
    \mathbb{E}_{b_1 \sim p(b_1|S^k)} \left[ f(b_1) \right] = \sum_{s \in S^k} P(b_1 = s|S^k) f(s) = \sum_{s \in S^k} p(s) R(S^k, s) f(s).
\end{equation}%
\end{proof}
The implication of this theorem is that the unordered set estimator, in explicit form given by \eqref{eq:unord_set_estimator}, is an unbiased estimator of $\mathbb{E}[f(x)]$ since it is the Rao-Blackwellized version of the unbiased single sample estimator. Also, as expected by taking multiple samples, it has variance equal or lower than the single sample estimator by the Rao-Blackwell Theorem \citep{lehmann1950completeness}.

\subsection{Rao-Blackwellization of other estimators}
\label{sec:rao_bw_other}
The unordered set estimator is also the result of Rao-Blackwellizing two other unbiased estimators: the \emph{stochastic sum-and-sample} estimator and the \emph{importance-weighted estimator}.

\paragraph{The sum-and-sample estimator.}
We define as \emph{sum-and-sample estimator} any estimator that relies on the identity that for any $C \subset D$
\begin{equation}
    \label{eq:sas_unbiased}
    \mathbb{E}_{x \sim p(x)}[f(x)] = \mathbb{E}_{x \sim p^{D \setminus C}(x)}\left[ \sum_{c \in C} p(c) f(c) + \left(1 - \sum_{c \in C} p(c) \right) f(x)\right].
\end{equation}
For the derivation, see Appendix \ref{app:proof_sas_unbiased} or \citet{liang2018memory,liu2019rao}. In general, a sum-and-sample estimator with a budget of $k > 1$ evaluations sums expectation terms for a set of categories $C$ (s.t. $|C| < k$) explicitly (e.g.\ selected by their value $f$ \citep{liang2018memory} or probability $p$ \citep{liu2019rao}), and uses $k - |C|$ (down-weighted) samples from $D \setminus C$ to estimate the remaining terms. As is noted by \citet{liu2019rao}, selecting $C$ such that $\frac{1 - \sum_{c \in C} p(c)}{k - |C|}$ is minimized guarantees to reduce variance compared to a standard minibatch of $k$ samples (which is equivalent to setting $C = \emptyset$).
See also \citet{fearnhead2003line} for a discussion on selecting $C$ optimally.
The ability to optimize $C$ depends on whether $p(c)$ can be computed efficiently a-priori (before sampling). This is difficult in high-dimensional settings, e.g.\ sequence models which compute the probability incrementally while ancestral sampling. An alternative is to select $C$ stochastically (as \eqref{eq:sas_unbiased} holds for any $C$), and we choose $C = B^{k-1}$ to define the \emph{stochastic sum-and-sample} estimator:
\begin{equation}
\label{eq:ssas_estimator}
    e^{\text{SSAS}}(B^k) = \sum_{j = 1}^{k-1} p(b_j) f(b_j) + \left(1 - \sum_{j=1}^{k-1} p(b_j)\right) f(b_k).
\end{equation}
For simplicity, we consider the version that sums $k - 1$ terms here, but the following results also hold for a version that sums $k - m$ terms and uses $m$ samples (without replacement) (see Appendix \ref{app:proof_sas_multi}). Sampling without replacement, it holds that $b_k|B^{k-1} \sim p^{D \setminus B^{k-1}}$, so the unbiasedness follows from \eqref{eq:sas_unbiased} by separating the expectation over $B^k$ into expectations over  $B^{k-1}$ and $b_k|B^{k-1}$:
\begin{equation*}
    \mathbb{E}_{B^{k-1} \sim p(B^{k-1})} \left[ \mathbb{E}_{b_k \sim p(b_k|B^{k-1})} \left[ e^{\text{SSAS}}(B^k) \right] \right] \\
    = \mathbb{E}_{B^{k-1} \sim p(B^{k-1})} \left[ \mathbb{E}[f(x)] \right] \\
    = \mathbb{E}[f(x)]. \\
\end{equation*}
In general, a sum-and-sample estimator reduces variance if the probability mass is concentrated on the summed categories. As typically high probability categories are sampled first, the stochastic sum-and-sample estimator sums high probability categories, similar to the estimator by \citet{liu2019rao} which we refer to as the \emph{deterministic sum-and-sample estimator}. As we show in Appendix \ref{app:proof_sas_rao}, Rao-Blackwellizing the stochastic sum-and-sample estimator also results in the unordered set estimator. This even holds for a version that uses $m$ samples and $k - m$ summed terms (see Appendix \ref{app:proof_sas_multi}), which means that the unordered set estimator has equal or lower variance than the optimal (in terms of $m$) stochastic sum-and-sample estimator, but conveniently does \emph{not} need to choose $m$.

\paragraph{The importance-weighted estimator.}
The importance-weighted estimator \citep{vieira2017estimating} is
\begin{equation}
\label{eq:iw_estimator}
    e^{\text{IW}}(S^k, \kappa) = \sum_{s \in S^k} \frac{p(s)}{q(s,\kappa)} f(s).
\end{equation}
This estimator is based on the idea of priority sampling \citep{duffield2007priority}. It does \emph{not} use the order of the sample, but assumes sampling using the Gumbel-Top-$k$ trick and requires access to $\kappa$, the $(k+1)$-th largest perturbed log-probability, which can be seen as the `threshold' since $g_{\phi_s} > \kappa \; \forall s \in S^k$.
$q(s, a) = P(g_{\phi_s} > a)$ can be interpreted as the \emph{inclusion probability} of $s \in S^k$ (assuming a fixed threshold $a$ instead of a fixed sample size $k$). For details and a proof of unbiasedness, see \citet{vieira2017estimating} or \citet{kool2019stochastic}.
As the estimator has high variance, \citet{kool2019stochastic} resort to \emph{normalizing} the importance weights, resulting in biased estimates. Instead, we use Rao-Blackwellization to eliminate stochasticity by $\kappa$. Again, the result is the unordered set estimator (see Appendix \ref{app:proof_iw_rao}), which thus has equal or lower variance.

\subsection{The unordered set policy gradient estimator}
Writing $p_{\bm{\theta}}$ to indicate the dependency on the model parameters $\bm{\theta}$, we can combine the unordered set estimator with REINFORCE \citep{williams1992simple} to obtain the \emph{unordered set policy gradient} estimator.
\begin{corollary}
\label{thm:us_unbiased}
The \emph{unordered set policy gradient estimator}, given by
\begin{equation}
\label{eq:unord_set_pg_estimator}
    e^{\text{USPG}}(S^k) = \sum_{s \in S^k} p_{\bm{\theta}}(s) R(S^k, s) \nabla_{\bm{\theta}} \log p_{\bm{\theta}}(s) f(s) = \sum_{s \in S^k} \nabla_{\bm{\theta}} p_{\bm{\theta}}(s) R(S^k, s) f(s),
\end{equation}
is an unbiased estimate of the policy gradient.
\end{corollary}
\begin{proof}
Using REINFORCE \citep{williams1992simple} combined with the unordered set estimator we find: 
\begin{equation*}
    \nabla_{\bm{\theta}} \mathbb{E}_{p_{\bm{\theta}}(x)}[f(x)] \hspace*{-0.06cm} 
    = \hspace*{-0.06cm} \mathbb{E}_{p_{\bm{\theta}}(x)}[\nabla_{\bm{\theta}} \log p_{\bm{\theta}}(x) f(x)] \hspace*{-0.06cm} = \hspace*{-0.06cm} \mathbb{E}_{S^k \sim p_{\bm{\theta}}(S^k)} \hspace*{-0.12cm} \left[ \hspace*{-0.06cm} \sum_{s \in S^k} p_{\bm{\theta}}(s) R(S^k, s) \nabla_{\bm{\theta}} \log p_{\bm{\theta}}(s) f(s)\hspace*{-0.06cm} \right]\hspace*{-0.12cm}.
\end{equation*}
\end{proof}

\paragraph{Variance reduction using a built-in control variate.}
The variance of REINFORCE can be reduced by subtracting a baseline from $f$. When taking multiple samples (with replacement), a simple and effective baseline is to take the mean of other (independent!) samples \citep{mnih2016variational}.
Sampling without replacement, we can use the same idea to construct a baseline based on the other samples, but we have to correct for the fact that the samples are \emph{not} independent.
\begin{theorem}
The \emph{unordered set policy gradient estimator with baseline}, given by
\begin{equation}
\label{eq:unord_set_pg_bl_estimator}
    e^{\text{USPGBL}}(S^k) = \sum_{s \in S^k} \nabla_{\bm{\theta}} p_{\bm{\theta}}(s) R(S^k, s) \left(f(s) - \sum_{s' \in S^k} p_{\bm{\theta}}(s') R^{D \setminus \{s\}}(S^k, s') f(s') \right),
\end{equation}
where 
\begin{equation}
\label{eq:second_order_leave_one_out_ratio}
    R^{D \setminus \{s\}}(S^k, s') = \frac{p_{\bm{\theta}}^{D \setminus \{s, s'\}}(S^k \setminus \{s, s'\})}{p_{\bm{\theta}}^{D \setminus \{s\}}(S^k \setminus \{s\})}
\end{equation}
is the \emph{second order leave-one-out ratio}, is an unbiased estimate of the policy gradient.
\end{theorem}
\begin{proof}
See Appendix \ref{app:proof_us_bl_unbiased}.
\end{proof}
This theorem shows how to include a built-in baseline based on \emph{dependent} samples (without replacement), without introducing bias. By having a built-in baseline, the value $f(s)$ for sample $s$ is compared against an estimate of its expectation $\mathbb{E}[f(s)]$, based on the other samples. The difference is an estimate of the \emph{advantage} \citep{sutton2018reinforcement}, which is positive if the sample $s$ is `better' than average, causing $p_{\bm{\theta}}(s)$ to be increased (reinforced) through the sign of the gradient, and vice versa. By sampling without replacement, the unordered set estimator forces the estimator to compare different alternatives, and reinforces the best among them.

\paragraph{Including the pathwise derivative.}
\label{sec:pathwise_gradient}
So far, we have only considered the scenario where $f$ does not depend on $\bm{\theta}$. If $f$ does depend on $\bm{\theta}$, for example in a VAE \citep{kingma2013auto,rezende2014stochastic}, then we use the notation $f_{\bm{\theta}}$ and we can write the gradient \citep{schulman2015gradient} as
\begin{equation}
\label{eq:reinforce_pathwise_identity}
    \nabla_{\bm{\theta}} \mathbb{E}_{p_{\bm{\theta}}(x)}[f_{\bm{\theta}}(x)] = \mathbb{E}_{p_{\bm{\theta}}(x)}[\nabla_{\bm{\theta}} \log p_{\bm{\theta}}(x) f_{\bm{\theta}}(x) + \nabla_{\bm{\theta}} f_{\bm{\theta}}(x)].
\end{equation}
The additional second (`pathwise') term can be estimated (using the same samples) with the standard unordered set estimator. This results in the \emph{full} unordered set policy gradient estimator:
\begin{align}
    e^{\text{FUSPG}}(S^k) &= \sum_{s \in S^k} \nabla_{\bm{\theta}} p_{\bm{\theta}}(s) R(S^k, s) f_{\bm{\theta}}(s) + \sum_{s \in S^k} p_{\bm{\theta}}(s) R(S^k, s) \nabla_{\bm{\theta}} f_{\bm{\theta}}(s) \notag \\
    &= \sum_{s \in S^k} R(S^k, s) \nabla_{\bm{\theta}} \left(p_{\bm{\theta}}(s) f_{\bm{\theta}}(s)\right) \label{eq:full_unord_set_pg_estimator_autograd}
\end{align}
\Eqref{eq:full_unord_set_pg_estimator_autograd} is straightforward to implement using an automatic differentiation library. We can also include the baseline (as in \eqref{eq:unord_set_pg_bl_estimator}) but we must make sure to call \textsc{stop\_gradient} (\textsc{detach} in PyTorch) on the baseline (but not on $f_{\bm{\theta}}(s)$!).
Importantly, we should \emph{never} track gradients through the leave-one-out ratio $R(S^k, s)$ which means it can be efficiently computed in pure inference mode.

\paragraph{Scope \& limitations.}
We can use the unordered set estimator for any discrete distribution from which we can sample without replacement, by treating it as a univariate categorical distribution over its domain.
This includes sequence models, from which we can sample using Stochastic Beam Search \citep{kool2019stochastic}, as well as multivariate categorical distributions which can also be treated as sequence models (see Section \ref{sec:experiments_vae}).
In the presence of continuous variables or a stochastic function $f$, we may separate this stochasticity from the stochasticity over the discrete distribution, as in \citet{lorberbom2019direct}.
The computation of the leave-one-out ratios adds some overhead, although they can be computed efficiently, even for large $k$ (see Appendix \ref{sec:computation_P_S}). For a moderately sized model, the costs of model evaluation and backpropagation dominate the cost of computing the estimator.

\subsection{Relation to other multi-sample estimators}
\paragraph{Relation to Murthy's estimator.}
We found out that the `vanilla' unordered set estimator (\eqref{eq:unord_set_estimator}) is actually a special case of the estimator by \citet{murthy1957ordered}, known in statistics literature for estimation of a population total $\Theta = \sum_{i \in D} y_i$. Using $y_i = p(i)f(i)$, we have $\Theta = \mathbb{E}[f(i)]$, so Murthy's estimator can be used to estimate expectations (see \eqref{eq:unord_set_estimator}). \citeauthor{murthy1957ordered} derives the estimator by `unordering' a convex combination of \citet{raj1956some} estimators, which, using $y_i = p(i)f(i)$, are stochastic sum-and-sample estimators in our analogy.

\citet{murthy1957ordered} also provides an unbiased estimator of the variance, which may be interesting for future applications. Since Murthy's estimator can be used with \emph{arbitrary} sampling distribution, it is straightforward to derive importance-sampling versions of our estimators. In particular, we can sample $S$ without replacement using $q(x) > 0, x \in D$, and use equations \ref{eq:unord_set_estimator}, \ref{eq:unord_set_pg_estimator}, \ref{eq:unord_set_pg_bl_estimator} and \ref{eq:full_unord_set_pg_estimator_autograd}, as long as we compute the leave-one-out ratio $R(S^k, s)$ using $q$.

While part of our derivation coincides with \citet{murthy1957ordered}, we are not aware of previous work using this estimator to estimate expectations. Additionally, we discuss practical computation of $p(S)$ (Appendix \ref{sec:computation_P_S}), we show the relation to the importance-weighted estimator, and we provide the extension to estimating policy gradients, especially including a built-in baseline without adding bias.

\paragraph{Relation to the empirical risk estimator.}
The empirical risk loss \citep{edunov2018classical} estimates the expectation in \eqref{eq:expectation} by summing only a subset $S$ of the domain, using \emph{normalized} probabilities $\hat{p}_{\bm{\theta}}(s) = \frac{p_{\bm{\theta}}(s)}{\sum_{s' \in S} p_{\bm{\theta}}(s)}$. Using this loss, the (biased) estimate of the gradient is given by
\begin{equation}
    \label{eq:risk}
    e^{\text{RISK}}(S^k) = \sum_{s \in S^k} \nabla_{\bm{\theta}} \left(\frac{p_{\bm{\theta}}(s)}{\sum_{s' \in S^k} p_{\bm{\theta}}(s')}\right) f(s).
\end{equation}
The risk estimator is similar to the unordered set policy gradient estimator, with two important differences: 1) the individual terms are normalized by the total probability mass rather than the leave-one-out ratio and 2) the gradient w.r.t.\ the normalization factor is taken into account. As a result, samples `compete' for probability mass and only the best can be reinforced. This has the same effect as using a built-in baseline, which we prove in the following theorem.
\begin{theorem}
    \label{thm:risk_baseline}
    By taking the gradient w.r.t. the normalization factor into account, the risk estimator has a built-in baseline, which means it can be written as
    \begin{equation}
    \label{eq:risk_bl}
        e^{\text{RISK}}(S^k) = \sum_{s \in S^k} \nabla_{\bm{\theta}} p_{\bm{\theta}}(s) \frac{1}{\sum_{s'' \in S^k} p_{\bm{\theta}}(s'')} \left(f(s) - \sum_{s' \in S^k} p_{\bm{\theta}}(s') \frac{1}{\sum_{s'' \in S^k} p_{\bm{\theta}}(s'')} f(s') \right).
    \end{equation}
\end{theorem}
\begin{proof}
See Appendix \ref{app:proof_risk_baseline}
\end{proof}
This theorem highlights the similarity between the biased risk estimator and our unbiased estimator (\eqref{eq:unord_set_pg_bl_estimator}), and suggests that their only difference is the weighting of terms. Unfortunately, the implementation by \citet{edunov2018classical} has more sources of bias (e.g.\ length normalization), which are not compatible with our estimator. However, we believe that our analysis helps analyze the bias of the risk estimator and is a step towards developing unbiased estimators for structured prediction.

\paragraph{Relation to VIMCO.}
VIMCO \citep{mnih2016variational} is an estimator that uses $k$ samples (with replacement) to optimize an objective of the form $\log \frac{1}{k} \sum_i f(x_i)$, which is a multi-sample stochastic lower bound in the context of variational inference. VIMCO reduces the variance by using a \emph{local} baseline for each of the $k$ samples, based on the other $k - 1$ samples. While we do not have a log term, as our goal is to optimize general $\mathbb{E}[f(x)]$, we adopt the idea of forming a baseline based on the other samples, and we define \emph{REINFORCE without replacement} (with built-in baseline) as the estimator that computes the gradient estimate using samples with replacement $X^k = (x_1, ..., x_k)$ as
\begin{equation}
\label{eq:reinforce_with_replacement}
    e^{\text{RFWR}}(X^k) = \frac{1}{k} \sum_{i = 1}^{k} \nabla_{\bm{\theta}} \log p_{\bm{\theta}}(x_i) \left(f(x_i) - \frac{1}{k-1} \sum_{j \neq i} f(x_j)\right).
\end{equation}
This estimator is unbiased, as $\mathbb{E}_{x_i, x_j}[\nabla_{\bm{\theta}} \log p_{\bm{\theta}}(x_i) f(x_j)] = 0$ for $i \neq j$ (see also \citet{kool2019buy}). We think of the unordered set estimator as the without-replacement version of this estimator, which weights terms by $p_{\bm{\theta}}(s) R(S^k, s)$ instead of $\frac{1}{k}$.
This puts more weight on higher probability elements to compensate for sampling without replacement. If probabilities are small and (close to) uniform, there are (almost) no duplicate samples and the weights will be close to $\frac{1}{k}$, so the gradient estimate of the with- and without-replacement versions are similar.

\paragraph{Relation to ARSM.}
ARSM \citep{yin2019arsm} also uses multiple evaluations (`pseudo-samples') of $p_{\bm{\theta}}$ and $f$. This can be seen as similar to sampling without replacement, and the estimator also has a built-in control variate. Compared to ARSM, our estimator allows direct control over the computational cost (through the sample size $k$) and has wider applicability, for example it also applies to multivariate categorical variables with different numbers of categories per dimension.

\paragraph{Relation to stratified/systematic sampling.}
Our estimator aims to reduce variance by changing the sampling distribution for multiple samples by sampling without replacement. There are alternatives, such as using stratified or systematic sampling (see, e.g.\ \citet{douc2005comparison}). Both partition the domain $D$ into $k$ strata and take a single sample from each stratum, where systematic sampling uses common random numbers for each stratum. In applications involving high-dimensional or structured domains, it is unclear how to partition the domain and how to sample from each partition. Additionally, as samples are \emph{not} independent, it is non-trivial to include a built-in baseline, which we find is a key component that makes our estimator perform well.

\section{Experiments}
\subsection{Bernoulli toy experiment}
\begin{figure}[t]
\vskip -0.4cm
    \centering
    \begin{subfigure}[b]{0.48\textwidth}
        \includegraphics[width=\textwidth]{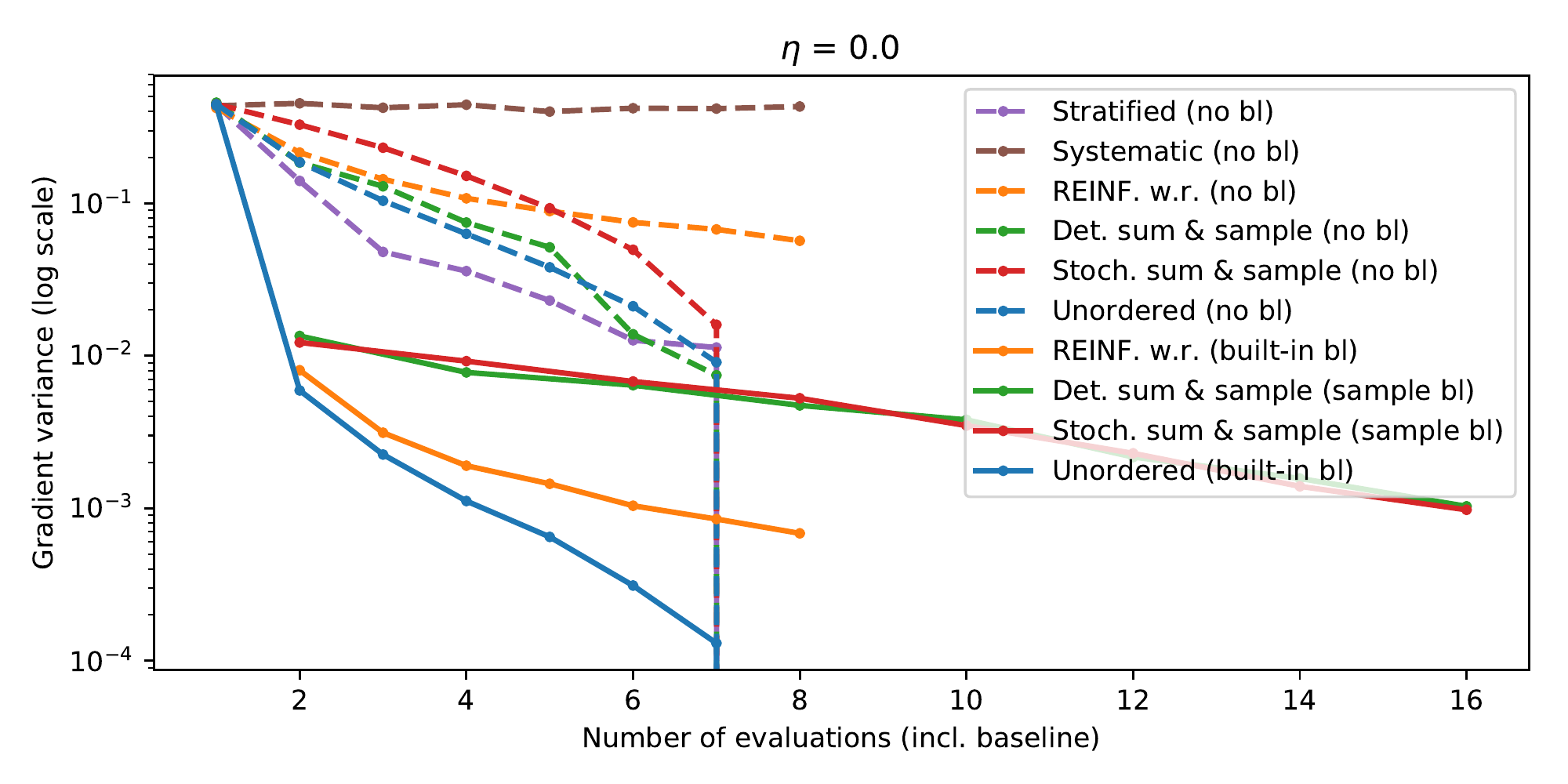}
        \vskip -0.2cm
        \caption{High entropy ($\eta = 0$)}
    \end{subfigure}
    \begin{subfigure}[b]{0.48\textwidth}
        \includegraphics[width=\textwidth]{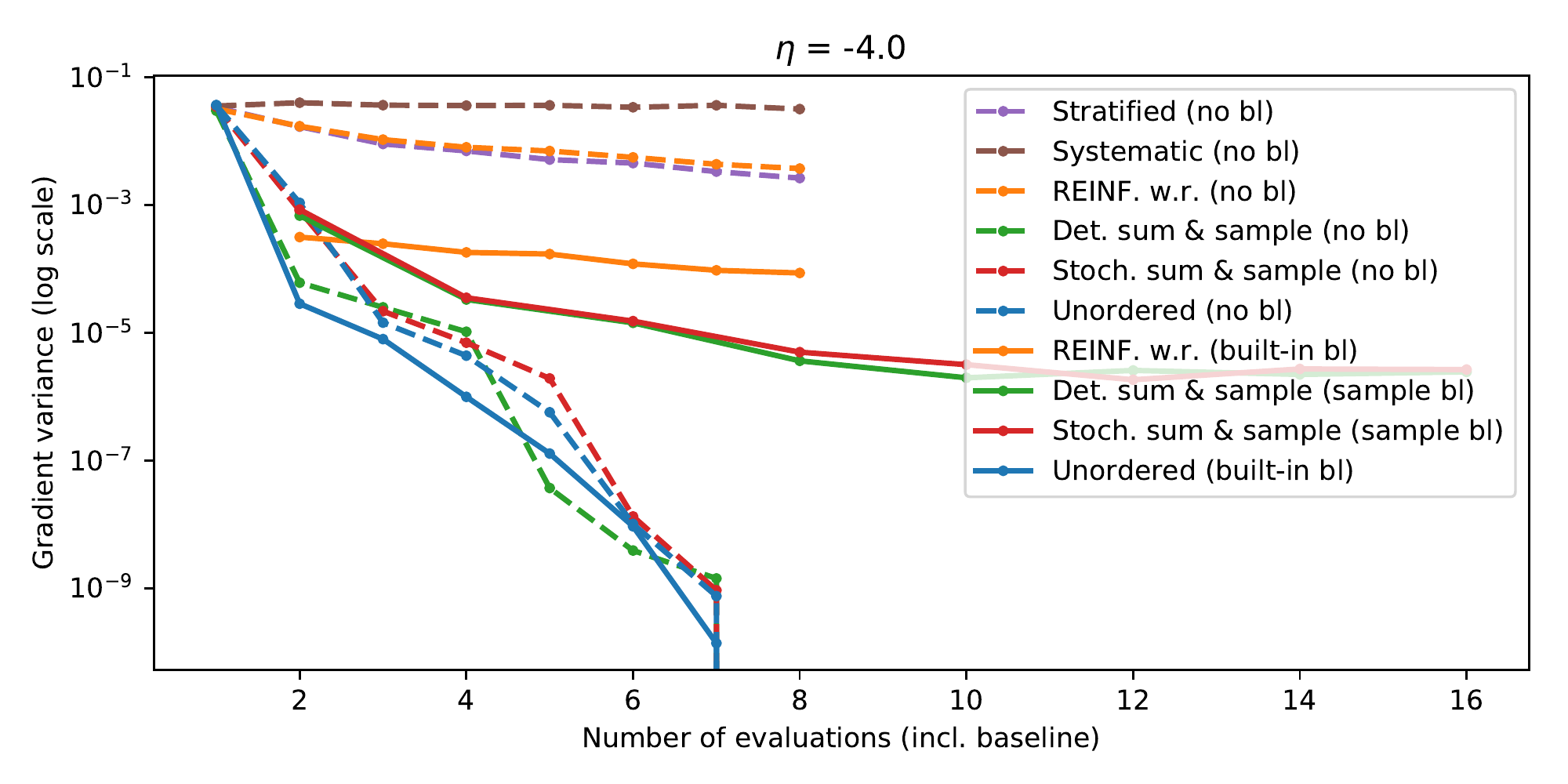}
        \vskip -0.2cm
        \caption{Low entropy ($\eta = -4$)}
    \end{subfigure}
\vskip -0.2cm
        \caption{Bernoulli gradient variance (on log scale) as a function of the number of model evaluations (including baseline evaluations, so the sum-and-sample estimators with sampled baselines use twice as many evaluations). Note that for some estimators, the variance is 0 (log variance $-\infty$) for $k = 8$.}
        \label{fig:bernoulli_toy}
\vskip -0.2cm
\end{figure}
We use the code by \citet{liu2019rao} to reproduce their Bernoulli toy experiment. Given a vector $\mathbf{p} = (0.6, 0.51, 0.48)$ the goal is to minimize the loss $\mathcal{L}(\eta) = \mathbb{E}_{x_1, x_2, x_3 \sim \text{Bern}(\sigma(\eta))} \left[ \sum_{i=1}^3 (x_i - p_i)^2 \right]$.
Here $x_1, x_2, x_3$ are i.i.d. from the $\text{Bernoulli}(\sigma(\eta))$ distribution, parameterized by a scalar $\eta \in \mathbb{R}$, where $\sigma(\eta) = (1 + \exp(-\eta))^{-1}$ is the sigmoid function.
We compare different estimators, with and without baseline (either `built-in' or using additional samples, referred to as REINFORCE+ in \citet{liu2019rao}). We report the (log-)variance of the scalar gradient $\frac{\partial \mathcal{L}}{\partial \eta}$ as a function of the number of model evaluations, which is twice as high when using a sampled baseline (for each term).

As can be seen in Figure \ref{fig:bernoulli_toy}, the unordered set estimator is the only estimator that has consistently the lowest (or comparable) variance in both the high ($\eta = 0$) and low entropy ($\eta = -4$) regimes and for different number of samples/model evaluations. This suggests that it combines the advantages of the other estimators. We also ran the actual optimization experiment, where with as few as $k=3$ samples the trajectory was indistinguishable from using the exact gradient (see \citet{liu2019rao}).

\subsection{Categorical Variational Auto-Encoder}
\label{sec:experiments_vae}
We use the code from \citet{yin2019arsm} to train a \emph{categorical} Variational Auto-Encoder (VAE) with 20 dimensional latent space, with 10 categories per dimension (details in Appendix \ref{app:categorical_vae_details}).
To use our estimator, we treat this as a single factorized distribution with $10^{20}$ categories from which we can sample without replacement using Stochastic Beam Search \citep{kool2019stochastic}, sequentially sampling each dimension as if it were a sequence model. We also perform experiments with $10^2$ latent space, which provides a lower entropy setting, to highlight the advantage of our estimator.

\paragraph{Measuring the variance.}
\begin{table}[b]
    \vskip -0.4cm
\small
    \caption{VAE gradient log-variance of different unbiased estimators with $k = 4$ samples.}
    \label{tab:vae_grads}
    \centering
    \setlength{\tabcolsep}{0.3em}
    \begin{tabular}{ll|cccccccc}
     & & ARSM & RELAX & \multicolumn{2}{c}{REINFORCE} & \multicolumn{2}{c}{Sum \& sample} & REINF. w.r. & Unordered \\
     \multicolumn{2}{c|}{Domain} & & & (no bl) & (sample bl) & (no bl) & (sample bl) & (built-in bl) & (built-in bl) \\
     \hline
     Small & $10^2$ & 13.45 & 11.67 & 11.52 & 7.49 & \bf 6.29 & \bf 6.29 & 6.65 & \bf 6.29 \\
     Large & $10^{20}$ & 15.55 & 15.86 & 13.81 & 8.48 & 13.77 & 8.44 & \bf 7.06 & \bf 7.05
    \end{tabular}
    \vskip -0.4cm
\end{table}
In Table \ref{tab:vae_grads}, we report the variance of different gradient estimators with $k = 4$ samples, evaluated on a trained model. The unordered set estimator has the lowest variance in both the small and large domain (low and high entropy) setting, being on-par with the best of the (stochastic\footnote{We cannot use the deterministic version by \citet{liu2019rao} since we cannot select the top $k$ categories.}) sum-and-sample estimator and REINFORCE with replacement\footnote{We cannot compare against VIMCO \citep{mnih2016variational} as it optimizes a different objective.}. This confirms the toy experiment, suggesting that the unordered set estimator provides the best of both estimators. In Appendix \ref{app:categorical_vae_extra_results} we repeat the same experiment at different stages of training, with similar results.

\clearpage

\begin{figure}[t]
    \centering
    \begin{subfigure}[b]{0.48\textwidth}
        \includegraphics[width=\textwidth]{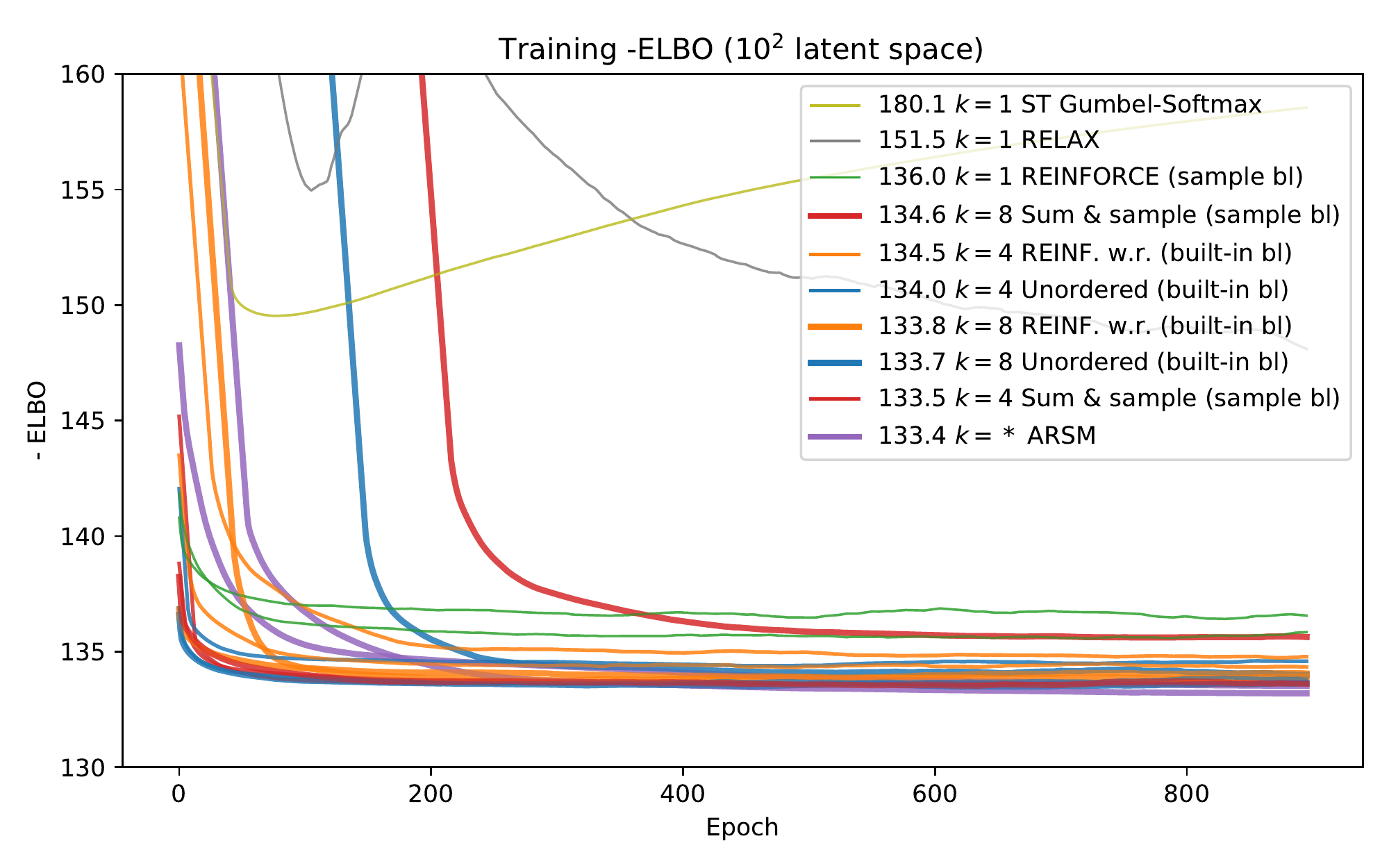}
        \caption{Small domain (latent space size $10^2$)}
    \end{subfigure}
    \begin{subfigure}[b]{0.48\textwidth}
        \includegraphics[width=\textwidth]{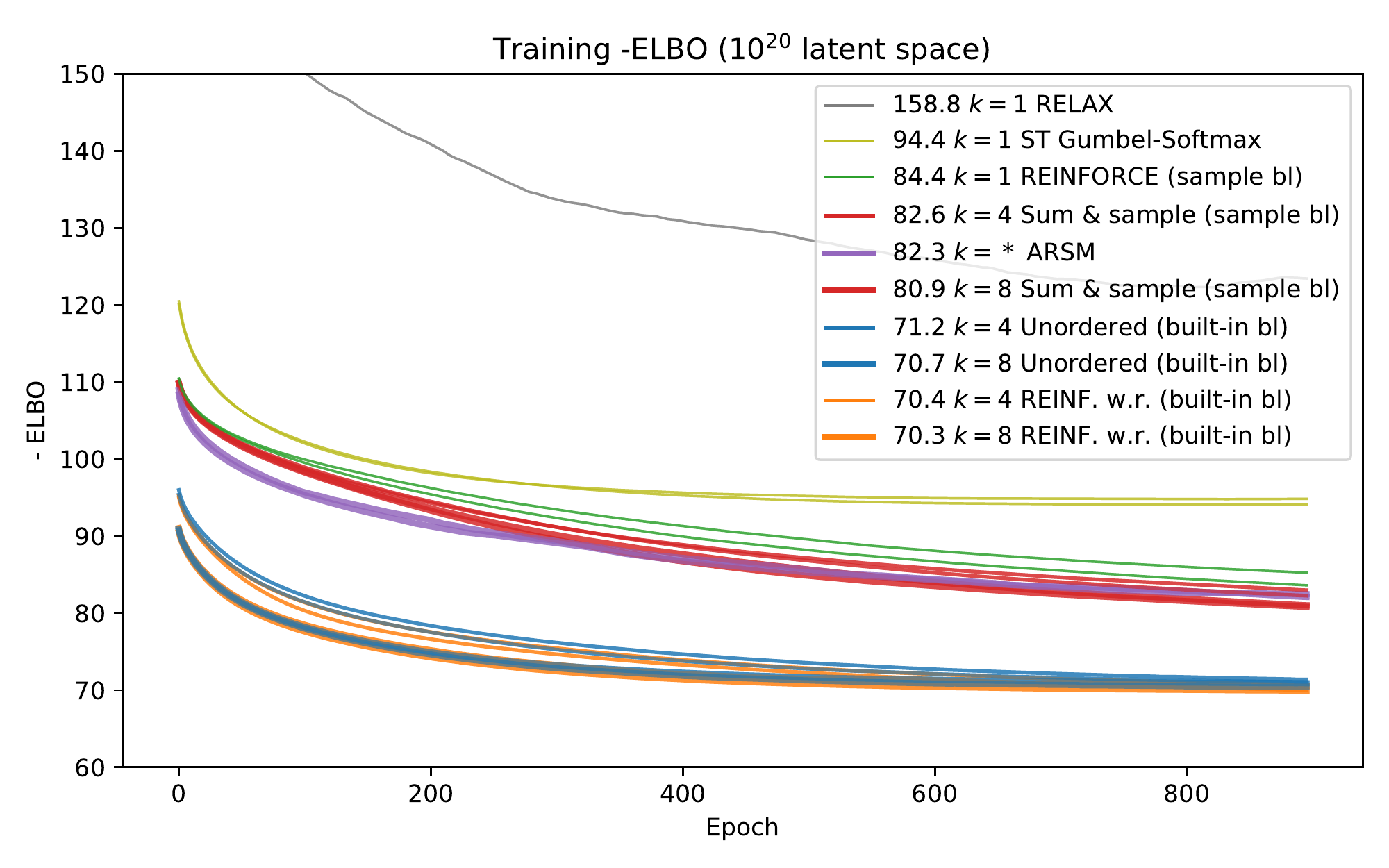}
        \caption{Large domain (latent space size $10^{20}$)}
    \end{subfigure}
        \caption{VAE smoothed training curves (-ELBO) of two independent runs when training with different estimators with $k =1$, 4 or 8 (thicker lines) samples (ARSM has a variable number). Some lines coincide, so we sort the legend by the lowest -ELBO achieved and report this value. }
        \label{fig:vae_train_elbo}
\end{figure}

\paragraph{ELBO optimization.}
We use different estimators to optimize the ELBO (details in Appendix \ref{app:categorical_vae_details}).
Additionally to the baselines by \citet{yin2019arsm} we compare against REINFORCE with replacement and the stochastic sum-and-sample estimator. In Figure \ref{fig:vae_train_elbo} we observe that our estimator performs on par with REINFORCE with replacement (and built-in baseline, \eqref{eq:reinforce_with_replacement}) and outperforms other estimators in at least one of the settings. There are a lot of other factors, e.g.\ exploration that may explain why we do not get a strictly better result despite the lower variance. We note some overfitting (see validation curves in Appendix \ref{app:categorical_vae_extra_results}), but since our goal is to show improved optimization, and to keep results directly comparable to \citet{yin2019arsm}, we consider regularization a separate issue outside the scope of this paper. These results are using MNIST binarized by a threshold of 0.5. In Appendix \ref{app:categorical_vae_extra_results} we report results using the standard binarized MNIST dataset from \citet{salakhutdinov2008quantitative}.

\subsection{Structured Prediction for the Travelling Salesman Problem}
To show the wide applicability of our estimator, we consider the structured prediction task of predicting routes (sequences) for the Travelling Salesman Problem (TSP) \citep{vinyals2015pointer,bello2016neural,kool2018attention}. We use the code by \citet{kool2018attention}\footnote{\url{https://github.com/wouterkool/attention-learn-to-route}} to reproduce their TSP experiment with 20 nodes. For details, see Appendix \ref{app:tsp_details}.

We implement REINFORCE with replacement (and built-in baseline) as well as the stochastic sum-and-sample estimator and our estimator, using Stochastic Beam Search \citep{kool2019stochastic} for sampling. Also, we include results using the biased normalized importance-weighted policy gradient estimator with built-in baseline (derived in \citet{kool2019buy}, see Appendix \ref{app:importance_weighted_pg}).
Additionally, we compare against REINFORCE with greedy rollout baseline \citep{rennie2017self} used by \citet{kool2019stochastic} and a batch-average baseline.
For reference, we also include the biased risk estimator, either `sampling' using stochastic or deterministic beam search (as in \citet{edunov2018classical}).

In Figure \ref{fig:steps_k4}, we compare training progress (measured on the validation set) as a function of the number of training steps, where we divide the batch size by $k$ to keep the total number of samples equal. Our estimator outperforms REINFORCE with replacement, the stochastic sum-and-sample estimator and the strong greedy rollout baseline (which uses additional baseline model evaluations) and performs on-par with the biased risk estimator. In Figure \ref{fig:instances_k4}, we plot the same results against the number of instances, which shows that, compared to the single sample estimators, we can train with less data and less computational cost (as we only need to run the encoder once for each instance).

\clearpage

\begin{figure*}[t]
\vskip -0.2cm
    \centering
    \begin{subfigure}[b]{0.45\textwidth}
        \includegraphics[width=\textwidth]{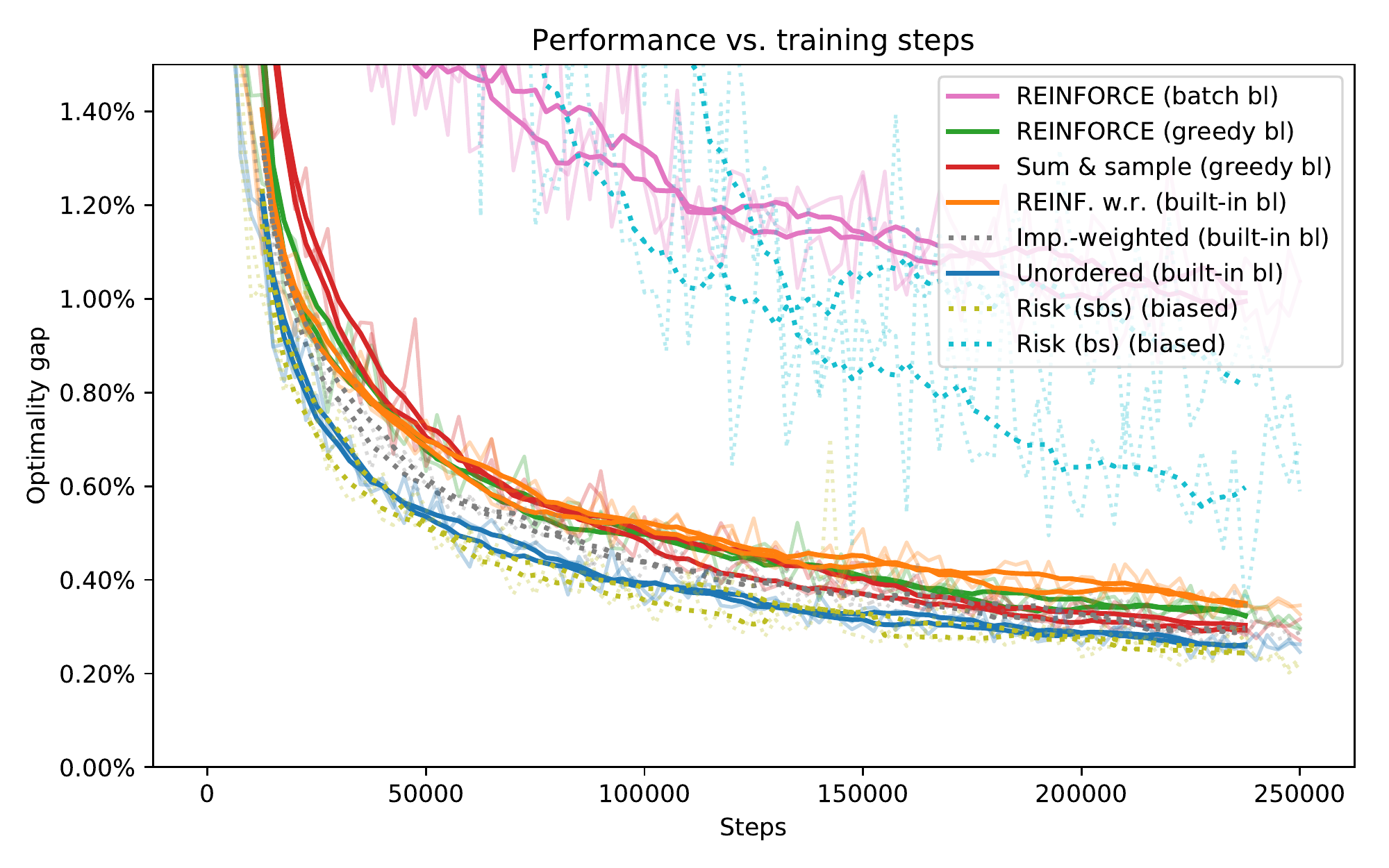}
\vskip -0.2cm
        \caption{Performance vs. training steps}
        \label{fig:steps_k4}
    \end{subfigure}
    ~ 
    \begin{subfigure}[b]{0.45\textwidth}
        \includegraphics[width=\textwidth]{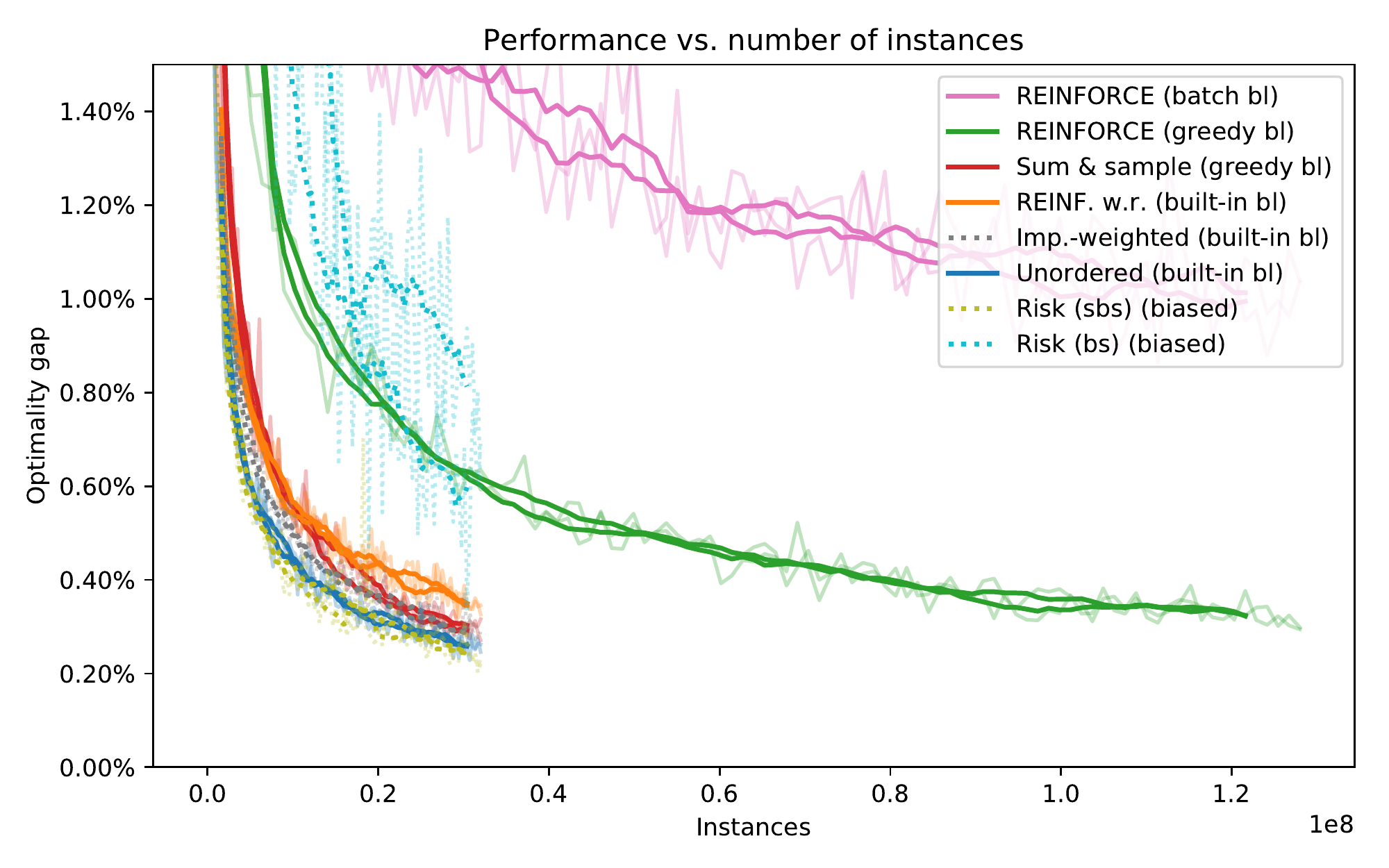}
\vskip -0.2cm
        \caption{Performance vs. number of instances}
        \label{fig:instances_k4}
    \end{subfigure}
\vskip -0.2cm
    \caption{TSP validation set optimality gap measured during training. Raw results are light, smoothed results are darker (2 random seeds). We compare our estimator against different unbiased and biased (dotted) multi-sample estimators and against single-sample REINFORCE, with batch-average or greedy rollout baseline.}
    \label{fig:results}
\vskip -0.2cm
\end{figure*}

\section{Discussion}
We introduced the unordered set estimator, a low-variance, unbiased gradient estimator based on sampling without replacement, which can be used as an alternative to the popular biased Gumbel-Softmax estimator \citep{jang2016categorical,maddison2016concrete}. Our estimator is the result of Rao-Blackwellizing three existing estimators, which guarantees equal or lower variance, and is closely related to a number of other estimators. It has wide applicability, is parameter free (except for the sample size $k$) and has competitive performance to the best of alternatives in both high and low entropy regimes.

In our experiments, we found that REINFORCE \emph{with} replacement, with multiple samples and a built-in baseline as inspired by VIMCO \citep{mnih2016variational}, is a simple yet strong estimator which has performance similar to our estimator in the high entropy setting. We are not aware of any recent work on gradient estimators for discrete distributions that has considered this estimator as baseline, while it may be often preferred given its simplicity.
In future work, we want to investigate if we can apply our estimator to estimate gradients `locally' \citep{titsias2015local}, as locally we have a smaller domain and expect more duplicate samples.

\subsubsection*{Acknowledgments}
This research was funded by ORTEC. We would like to thank anonymous reviewers for their feedback that helped improve the paper.

\bibliography{iclr2020_conference}

\begin{thebibliography}{56}
\providecommand{\natexlab}[1]{#1}
\providecommand{\url}[1]{\texttt{#1}}
\expandafter\ifx\csname urlstyle\endcsname\relax
  \providecommand{\doi}[1]{doi: #1}\else
  \providecommand{\doi}{doi: \begingroup \urlstyle{rm}\Url}\fi

\bibitem[Bahdanau et~al.(2017)Bahdanau, Brakel, Xu, Goyal, Lowe, Pineau,
  Courville, and Bengio]{bahdanau2017actor}
Dzmitry Bahdanau, Philemon Brakel, Kelvin Xu, Anirudh Goyal, Ryan Lowe, Joelle
  Pineau, Aaron Courville, and Yoshua Bengio.
\newblock An actor-critic algorithm for sequence prediction.
\newblock In \emph{International Conference on Learning Representations}, 2017.

\bibitem[Bello et~al.(2016)Bello, Pham, Le, Norouzi, and
  Bengio]{bello2016neural}
Irwan Bello, Hieu Pham, Quoc~V Le, Mohammad Norouzi, and Samy Bengio.
\newblock Neural combinatorial optimization with reinforcement learning.
\newblock \emph{arXiv preprint arXiv:1611.09940}, 2016.

\bibitem[Bengio et~al.(2013)Bengio, L{\'e}onard, and
  Courville]{bengio2013estimating}
Yoshua Bengio, Nicholas L{\'e}onard, and Aaron Courville.
\newblock Estimating or propagating gradients through stochastic neurons for
  conditional computation.
\newblock \emph{arXiv preprint arXiv:1308.3432}, 2013.

\bibitem[Casella \& Robert(1996)Casella and Robert]{casella1996rao}
George Casella and Christian~P Robert.
\newblock Rao-{B}lackwellisation of sampling schemes.
\newblock \emph{Biometrika}, 83\penalty0 (1):\penalty0 81--94, 1996.

\bibitem[Douc \& Capp{\'e}(2005)Douc and Capp{\'e}]{douc2005comparison}
Randal Douc and Olivier Capp{\'e}.
\newblock Comparison of resampling schemes for particle filtering.
\newblock In \emph{ISPA 2005. Proceedings of the 4th International Symposium on
  Image and Signal Processing and Analysis, 2005.}, pp.\  64--69. IEEE, 2005.

\bibitem[Duffield et~al.(2007)Duffield, Lund, and Thorup]{duffield2007priority}
Nick Duffield, Carsten Lund, and Mikkel Thorup.
\newblock Priority sampling for estimation of arbitrary subset sums.
\newblock \emph{Journal of the ACM (JACM)}, 54\penalty0 (6):\penalty0 32, 2007.

\bibitem[Edunov et~al.(2018)Edunov, Ott, Auli, Grangier,
  et~al.]{edunov2018classical}
Sergey Edunov, Myle Ott, Michael Auli, David Grangier, et~al.
\newblock Classical structured prediction losses for sequence to sequence
  learning.
\newblock In \emph{Proceedings of the 2018 Conference of the North American
  Chapter of the Association for Computational Linguistics: Human Language
  Technologies, Volume 1 (Long Papers)}, volume~1, pp.\  355--364, 2018.

\bibitem[Fearnhead \& Clifford(2003)Fearnhead and Clifford]{fearnhead2003line}
Paul Fearnhead and Peter Clifford.
\newblock On-line inference for hidden markov models via particle filters.
\newblock \emph{Journal of the Royal Statistical Society: Series B (Statistical
  Methodology)}, 65\penalty0 (4):\penalty0 887--899, 2003.

\bibitem[Grathwohl et~al.(2018)Grathwohl, Choi, Wu, Roeder, and
  Duvenaud]{grathwohl2017backpropagation}
Will Grathwohl, Dami Choi, Yuhuai Wu, Geoffrey Roeder, and David Duvenaud.
\newblock Backpropagation through the void: Optimizing control variates for
  black-box gradient estimation.
\newblock In \emph{International Conference on Learning Representations}, 2018.

\bibitem[Gregor et~al.(2014)Gregor, Danihelka, Mnih, Blundell, and
  Wierstra]{gregor2014deep}
Karol Gregor, Ivo Danihelka, Andriy Mnih, Charles Blundell, and Daan Wierstra.
\newblock Deep autoregressive networks.
\newblock In \emph{International Conference on Machine Learning}, pp.\
  1242--1250, 2014.

\bibitem[Grover et~al.(2019)Grover, Wang, Zweig, and
  Ermon]{grover2018stochastic}
Aditya Grover, Eric Wang, Aaron Zweig, and Stefano Ermon.
\newblock Stochastic optimization of sorting networks via continuous
  relaxations.
\newblock In \emph{International Conference on Learning Representations}, 2019.

\bibitem[Gu et~al.(2018)Gu, Im, and Li]{gu2017neural}
Jiatao Gu, Daniel~Jiwoong Im, and Victor~OK Li.
\newblock Neural machine translation with {Gumbel}-greedy decoding.
\newblock In \emph{Thirty-Second AAAI Conference on Artificial Intelligence
  (AAAI)}, 2018.

\bibitem[Gu et~al.(2016)Gu, Levine, Sutskever, and Mnih]{gu2015muprop}
Shixiang Gu, Sergey Levine, Ilya Sutskever, and Andriy Mnih.
\newblock Muprop: Unbiased backpropagation for stochastic neural networks.
\newblock In \emph{International Conference on Learning Representations}, 2016.

\bibitem[He et~al.(2016)He, Xia, Qin, Wang, Yu, Liu, and Ma]{he2016dual}
Di~He, Yingce Xia, Tao Qin, Liwei Wang, Nenghai Yu, Tie-Yan Liu, and Wei-Ying
  Ma.
\newblock Dual learning for machine translation.
\newblock In \emph{Advances in Neural Information Processing Systems}, pp.\
  820--828, 2016.

\bibitem[Jang et~al.(2016)Jang, Gu, and Poole]{jang2016categorical}
Eric Jang, Shixiang Gu, and Ben Poole.
\newblock Categorical reparameterization with gumbel-softmax.
\newblock In \emph{International Conference on Learning Representations}, 2016.

\bibitem[Kim et~al.(2016)Kim, Sabharwal, and Ermon]{kim2016exact}
Carolyn Kim, Ashish Sabharwal, and Stefano Ermon.
\newblock Exact sampling with integer linear programs and random perturbations.
\newblock In \emph{Thirtieth AAAI Conference on Artificial Intelligence}, 2016.

\bibitem[Kingma \& Ba(2015)Kingma and Ba]{kingma2015adam}
Diederik~P Kingma and Jimmy Ba.
\newblock Adam: A method for stochastic optimization.
\newblock In \emph{International Conference on Learning Representations}, 2015.

\bibitem[Kingma \& Welling(2014)Kingma and Welling]{kingma2013auto}
Diederik~P Kingma and Max Welling.
\newblock Auto-encoding variational {B}ayes.
\newblock In \emph{International Conference on Learning Representations}, 2014.

\bibitem[Kool et~al.(2019{\natexlab{a}})Kool, van Hoof, and
  Welling]{kool2018attention}
Wouter Kool, Herke van Hoof, and Max Welling.
\newblock Attention, learn to solve routing problems!
\newblock In \emph{International Conference on Learning Representations},
  2019{\natexlab{a}}.

\bibitem[Kool et~al.(2019{\natexlab{b}})Kool, van Hoof, and
  Welling]{kool2019buy}
Wouter Kool, Herke van Hoof, and Max Welling.
\newblock Buy 4 reinforce samples, get a baseline for free!
\newblock In \emph{Deep Reinforcement Learning Meets Structured Prediction
  Workshop at the International Conference on Learning Representations},
  2019{\natexlab{b}}.

\bibitem[Kool et~al.(2019{\natexlab{c}})Kool, Van~Hoof, and
  Welling]{kool2019stochastic}
Wouter Kool, Herke Van~Hoof, and Max Welling.
\newblock Stochastic beams and where to find them: The gumbel-top-k trick for
  sampling sequences without replacement.
\newblock In \emph{International Conference on Machine Learning}, pp.\
  3499--3508, 2019{\natexlab{c}}.

\bibitem[Larochelle \& Murray(2011)Larochelle and Murray]{larochelle2011neural}
Hugo Larochelle and Iain Murray.
\newblock The neural autoregressive distribution estimator.
\newblock In \emph{Proceedings of the Fourteenth International Conference on
  Artificial Intelligence and Statistics}, pp.\  29--37, 2011.

\bibitem[Leblond et~al.(2018)Leblond, Alayrac, Osokin, and
  Lacoste-Julien]{leblond2018searnn}
R{\'e}mi Leblond, Jean-Baptiste Alayrac, Anton Osokin, and Simon
  Lacoste-Julien.
\newblock Searnn: Training {RNN}s with global-local losses.
\newblock In \emph{6th International Conference on Learning Representations},
  2018.

\bibitem[Lehmann \& Scheff{\'e}(1950)Lehmann and
  Scheff{\'e}]{lehmann1950completeness}
EL~Lehmann and Henry Scheff{\'e}.
\newblock Completeness, similar regions, and unbiased estimation: Part i.
\newblock \emph{Sankhy{\=a}: The Indian Journal of Statistics}, pp.\  305--340,
  1950.

\bibitem[Liang et~al.(2018)Liang, Norouzi, Berant, Le, and
  Lao]{liang2018memory}
Chen Liang, Mohammad Norouzi, Jonathan Berant, Quoc~V Le, and Ni~Lao.
\newblock Memory augmented policy optimization for program synthesis and
  semantic parsing.
\newblock In \emph{Advances in Neural Information Processing Systems}, pp.\
  9994--10006, 2018.

\bibitem[Liu et~al.(2019)Liu, Regier, Tripuraneni, Jordan, and
  Mcauliffe]{liu2019rao}
Runjing Liu, Jeffrey Regier, Nilesh Tripuraneni, Michael Jordan, and Jon
  Mcauliffe.
\newblock Rao-{B}lackwellized stochastic gradients for discrete distributions.
\newblock In \emph{International Conference on Machine Learning}, pp.\
  4023--4031, 2019.

\bibitem[Lorberbom et~al.(2018)Lorberbom, Gane, Jaakkola, and
  Hazan]{lorberbom2018direct}
Guy Lorberbom, Andreea Gane, Tommi Jaakkola, and Tamir Hazan.
\newblock Direct optimization through argmax for discrete variational
  auto-encoder.
\newblock \emph{arXiv preprint arXiv:1806.02867}, 2018.

\bibitem[Lorberbom et~al.(2019)Lorberbom, Maddison, Heess, Hazan, and
  Tarlow]{lorberbom2019direct}
Guy Lorberbom, Chris~J Maddison, Nicolas Heess, Tamir Hazan, and Daniel Tarlow.
\newblock Direct policy gradients: Direct optimization of policies in discrete
  action spaces.
\newblock \emph{arXiv preprint arXiv:1906.06062}, 2019.

\bibitem[Luce(1959)]{luce1959individual}
R~Duncan Luce.
\newblock \emph{Individual choice behavior: A theoretical analysis}.
\newblock John Wiley, 1959.

\bibitem[Maddison et~al.(2014)Maddison, Tarlow, and
  Minka]{maddison2014sampling}
Chris~J Maddison, Daniel Tarlow, and Tom Minka.
\newblock A* sampling.
\newblock In \emph{Advances in Neural Information Processing Systems}, pp.\
  3086--3094, 2014.

\bibitem[Maddison et~al.(2016)Maddison, Mnih, and Teh]{maddison2016concrete}
Chris~J Maddison, Andriy Mnih, and Yee~Whye Teh.
\newblock The concrete distribution: A continuous relaxation of discrete random
  variables.
\newblock In \emph{International Conference on Learning Representations}, 2016.

\bibitem[Mnih \& Gregor(2014)Mnih and Gregor]{mnih2014neural}
Andriy Mnih and Karol Gregor.
\newblock Neural variational inference and learning in belief networks.
\newblock In \emph{International Conference on Machine Learning}, pp.\
  1791--1799, 2014.

\bibitem[Mnih \& Rezende(2016)Mnih and Rezende]{mnih2016variational}
Andriy Mnih and Danilo Rezende.
\newblock Variational inference for {M}onte {C}arlo objectives.
\newblock In \emph{International Conference on Machine Learning}, pp.\
  2188--2196, 2016.

\bibitem[Murthy(1957)]{murthy1957ordered}
MN~Murthy.
\newblock Ordered and unordered estimators in sampling without replacement.
\newblock \emph{Sankhy{\=a}: The Indian Journal of Statistics (1933-1960)},
  18\penalty0 (3/4):\penalty0 379--390, 1957.

\bibitem[Negrinho et~al.(2018)Negrinho, Gormley, and
  Gordon]{negrinho2018learning}
Renato Negrinho, Matthew Gormley, and Geoffrey~J Gordon.
\newblock Learning beam search policies via imitation learning.
\newblock In \emph{Advances in Neural Information Processing Systems}, pp.\
  10673--10682, 2018.

\bibitem[Norouzi et~al.(2016)Norouzi, Bengio, Jaitly, Schuster, Wu, Schuurmans,
  et~al.]{norouzi2016reward}
Mohammad Norouzi, Samy Bengio, Navdeep Jaitly, Mike Schuster, Yonghui Wu, Dale
  Schuurmans, et~al.
\newblock Reward augmented maximum likelihood for neural structured prediction.
\newblock In \emph{Advances In Neural Information Processing Systems}, pp.\
  1723--1731, 2016.

\bibitem[Paisley et~al.(2012)Paisley, Blei, and Jordan]{paisley2012variational}
John Paisley, David~M Blei, and Michael~I Jordan.
\newblock Variational {B}ayesian inference with stochastic search.
\newblock In \emph{International Conference on Machine Learning}, pp.\
  1363--1370, 2012.

\bibitem[Plackett(1975)]{plackett1975analysis}
Robin~L Plackett.
\newblock The analysis of permutations.
\newblock \emph{Journal of the Royal Statistical Society: Series C (Applied
  Statistics)}, 24\penalty0 (2):\penalty0 193--202, 1975.

\bibitem[Raj(1956)]{raj1956some}
Des Raj.
\newblock Some estimators in sampling with varying probabilities without
  replacement.
\newblock \emph{Journal of the American Statistical Association}, 51\penalty0
  (274):\penalty0 269--284, 1956.

\bibitem[Ranganath et~al.(2014)Ranganath, Gerrish, and
  Blei]{ranganath2014black}
Rajesh Ranganath, Sean Gerrish, and David Blei.
\newblock Black box variational inference.
\newblock In \emph{Artificial Intelligence and Statistics}, pp.\  814--822,
  2014.

\bibitem[Ranzato et~al.(2016)Ranzato, Chopra, Auli, and
  Zaremba]{ranzato2016sequence}
Marc'Aurelio Ranzato, Sumit Chopra, Michael Auli, and Wojciech Zaremba.
\newblock Sequence level training with recurrent neural networks.
\newblock In \emph{International Conference on Learning Representations}, 2016.

\bibitem[Rennie et~al.(2017)Rennie, Marcheret, Mroueh, Ross, and
  Goel]{rennie2017self}
Steven~J Rennie, Etienne Marcheret, Youssef Mroueh, Jerret Ross, and Vaibhava
  Goel.
\newblock Self-critical sequence training for image captioning.
\newblock In \emph{Proceedings of the IEEE Conference on Computer Vision and
  Pattern Recognition}, pp.\  7008--7024, 2017.

\bibitem[Rezende et~al.(2014)Rezende, Mohamed, and
  Wierstra]{rezende2014stochastic}
Danilo~Jimenez Rezende, Shakir Mohamed, and Daan Wierstra.
\newblock Stochastic backpropagation and approximate inference in deep
  generative models.
\newblock In \emph{International Conference on Machine Learning}, pp.\
  1278--1286, 2014.

\bibitem[Roeder et~al.(2017)Roeder, Wu, and Duvenaud]{roeder2017sticking}
Geoffrey Roeder, Yuhuai Wu, and David~K Duvenaud.
\newblock Sticking the landing: Simple, lower-variance gradient estimators for
  variational inference.
\newblock In \emph{Advances in Neural Information Processing Systems}, pp.\
  6925--6934, 2017.

\bibitem[Salakhutdinov \& Murray(2008)Salakhutdinov and
  Murray]{salakhutdinov2008quantitative}
Ruslan Salakhutdinov and Iain Murray.
\newblock On the quantitative analysis of deep belief networks.
\newblock In \emph{International Conference on Machine Learning}, pp.\
  872--879, 2008.

\bibitem[Schulman et~al.(2015)Schulman, Heess, Weber, and
  Abbeel]{schulman2015gradient}
John Schulman, Nicolas Heess, Theophane Weber, and Pieter Abbeel.
\newblock Gradient estimation using stochastic computation graphs.
\newblock In \emph{Advances in Neural Information Processing Systems}, pp.\
  3528--3536, 2015.

\bibitem[Shen et~al.(2016)Shen, Cheng, He, He, Wu, Sun, and
  Liu]{shen2016minimum}
Shiqi Shen, Yong Cheng, Zhongjun He, Wei He, Hua Wu, Maosong Sun, and Yang Liu.
\newblock Minimum risk training for neural machine translation.
\newblock In \emph{Proceedings of the 54th Annual Meeting of the Association
  for Computational Linguistics (Volume 1: Long Papers)}, volume~1, pp.\
  1683--1692, 2016.

\bibitem[Sutton \& Barto(2018)Sutton and Barto]{sutton2018reinforcement}
Richard~S Sutton and Andrew~G Barto.
\newblock \emph{Reinforcement learning: An introduction}.
\newblock MIT press, 2018.

\bibitem[Titsias \& L{\'a}zaro-Gredilla(2015)Titsias and
  L{\'a}zaro-Gredilla]{titsias2015local}
Michalis~K Titsias and Miguel L{\'a}zaro-Gredilla.
\newblock Local expectation gradients for black box variational inference.
\newblock In \emph{Advances in Neural Information Processing Systems-Volume 2},
  pp.\  2638--2646, 2015.

\bibitem[Tucker et~al.(2017)Tucker, Mnih, Maddison, Lawson, and
  Sohl-Dickstein]{tucker2017rebar}
George Tucker, Andriy Mnih, Chris~J Maddison, John Lawson, and Jascha
  Sohl-Dickstein.
\newblock Rebar: Low-variance, unbiased gradient estimates for discrete latent
  variable models.
\newblock In \emph{Advances in Neural Information Processing Systems}, pp.\
  2627--2636, 2017.

\bibitem[Vieira(2014)]{vieira2014gumbel}
Tim Vieira.
\newblock Gumbel-max trick and weighted reservoir sampling, 2014.
\newblock URL
  \url{https://timvieira.github.io/blog/post/2014/08/01/gumbel-max-trick-and-weighted-reservoir-sampling/}.

\bibitem[Vieira(2017)]{vieira2017estimating}
Tim Vieira.
\newblock Estimating means in a finite universe, 2017.
\newblock URL
  \url{https://timvieira.github.io/blog/post/2017/07/03/estimating-means-in-a-finite-universe/}.

\bibitem[Vinyals et~al.(2015)Vinyals, Fortunato, and
  Jaitly]{vinyals2015pointer}
Oriol Vinyals, Meire Fortunato, and Navdeep Jaitly.
\newblock Pointer networks.
\newblock In \emph{Advances in Neural Information Processing Systems}, pp.\
  2692--2700, 2015.

\bibitem[Williams(1992)]{williams1992simple}
Ronald~J Williams.
\newblock Simple statistical gradient-following algorithms for connectionist
  reinforcement learning.
\newblock \emph{Machine learning}, 8\penalty0 (3-4):\penalty0 229--256, 1992.

\bibitem[Yellott(1977)]{yellott1977relationship}
John~I Yellott.
\newblock The relationship between {L}uce's choice axiom, {T}hurstone's theory
  of comparative judgment, and the double exponential distribution.
\newblock \emph{Journal of Mathematical Psychology}, 15\penalty0 (2):\penalty0
  109--144, 1977.

\bibitem[Yin et~al.(2019)Yin, Yue, and Zhou]{yin2019arsm}
Mingzhang Yin, Yuguang Yue, and Mingyuan Zhou.
\newblock Arsm: Augment-reinforce-swap-merge estimator for gradient
  backpropagation through categorical variables.
\newblock In \emph{International Conference on Machine Learning}, pp.\
  7095--7104, 2019.

\end{thebibliography}
\bibliographystyle{iclr2020_conference}

\clearpage
\appendix
\section{Notation}
\label{app:notation}
Throughout this appendix we will use the following notation from \citet{maddison2014sampling}:
\begin{align*}
    e_\phi(g) &= \exp(-g + \phi) \\
    F_\phi(g) &= \exp(-\exp(-g+\phi)) \\
    f_\phi(g) &= e_\phi(g) F_\phi(g).
\end{align*}
This means that $F_\phi(g)$ is the CDF and $f_\phi(g)$ the PDF of the $\text{Gumbel}(\phi)$ distribution. Additionally we will use the identities by \citet{maddison2014sampling}:
\begin{align}
    F_\phi(g) F_\gamma(g) &= F_{\log(\exp(\phi) + \exp(\gamma))}(g) \\
    \int_{g=a}^b e_{\gamma}(g) F_\phi(g) \partial g &= (F_\phi(b) - F_\phi(a))\frac{\exp(\gamma)}{\exp(\phi)}.
\end{align}

Also, we will use the following notation, definitions and identities (see \citet{kool2019stochastic}):
\begin{align}
    \phi_i &= \log p(i) \\
    \phi_S &= \log \sum_{i \in S} p(i) = \log \sum_{i \in S} \exp \phi_i \\
    \phi_{D \setminus S} &= \log \sum_{i \in D \setminus S} p(i) = \log \left(1 - \sum_{i \in S} p(i) \right) = \log (1 - \exp(\phi_S)) \\
    G_{\phi_i} &\sim \text{Gumbel}(\phi_i) \\
    G_{\phi_S} &= \max_{i \in S} G_{\phi_i} \sim \text{Gumbel}(\phi_S) \label{eq:gumbel_max_trick}
\end{align}
For a proof of \eqref{eq:gumbel_max_trick}, see \citet{maddison2014sampling}.

\section{Computation of $p(S^k)$, $p^{D \setminus C}(S \setminus C)$ and $R(S^k, s)$}
\label{sec:computation_P_S}
We can sample the set $S^k$ from the Plackett-Luce distribution using the Gumbel-Top-$k$ trick by drawing Gumbel variables $G_{\phi_i} \sim \text{Gumbel}(\phi_i)$ for each element and returning the indices of the $k$ largest Gumbels. If we ignore the ordering, this means we will obtain the set $S^k$ if $\min_{i \in S^k} G_{\phi_i} > \max_{i \in D \setminus S^k} G_{\phi_i}$. Omitting the superscript $k$ for clarity, we can use the Gumbel-Max trick, i.e.\ that $G_{\phi_{D \setminus S}} = \max_{i \not \in S} G_{\phi_i} \sim \text{Gumbel}(\phi_{D \setminus S})$ (\eqref{eq:gumbel_max_trick}) and marginalize over $G_{\phi_{D \setminus S}}$:
\begin{align}
    p(S) &= P(\min_{i \in S} G_{\phi_i} > G_{\phi_{D \setminus S}}) \notag \\
    &= P(G_{\phi_i} > G_{\phi_{D \setminus S}}, i \in S) \notag \\
    &= \int_{g_{\phi_{D \setminus S}} = -\infty}^{\infty} f_{\phi_{D \setminus S}}(g_{\phi_{D \setminus S}}) P(G_{\phi_i} > g_{\phi_{D \setminus S}}, i \in S) \partial g_{\phi_{D \setminus S}} \notag \\
    &= \int_{g_{\phi_{D \setminus S}} = -\infty}^{\infty} f_{\phi_{D \setminus S}}(g_{\phi_{D \setminus S}})  \prod_{i \in S} \left(1 - F_{\phi_i}(g_{\phi_{D \setminus S}})\right) \partial g_{\phi_{D \setminus S}} \label{eq:P_S_indefinite_integral} \\
    &= \int_{u = 0}^{1} \prod_{i \in S} \left(1 - F_{\phi_i}\left(F^{-1}_{\phi_{D \setminus S}}(u)\right)\right) \partial u \label{eq:P_S_integral}
\end{align}
Here we have used a change of variables $u = F_{\phi_{D \setminus S}}(g_{\phi_{D \setminus S}})$. This expression can be efficiently numerically integrated (although another change of variables may be required for numerical stability depending on the values of $\bm{\phi}$).

\paragraph{Exact computation in $O(2^k)$.}
The integral in \eqref{eq:P_S_indefinite_integral} can be computed exactly using the identity
\begin{equation*}
    \prod_{i \in S} (a_i - b_i) = \sum_{C \subseteq S} (-1)^{|C|} \prod_{i \in C} b_i \prod_{i \in S \setminus C} a_i
\end{equation*}
which gives
\begin{align}
p(S) &= \int_{g_{\phi_{D \setminus S}} = -\infty}^{\infty} f_{\phi_{D \setminus S}}(g_{\phi_{D \setminus S}})  \prod_{i \in S} \left(1 - F_{\phi_i}(g_{\phi_{D \setminus S}})\right) \partial g_{\phi_{D \setminus S}}\notag \\
    &= \sum_{C \subseteq S} (-1)^{|C|} \int_{g_{\phi_{D \setminus S}} = -\infty}^{\infty} f_{\phi_{D \setminus S}}(g_{\phi_{D \setminus S}}) \prod_{i \in C} F_{\phi_i}(g_{\phi_{D \setminus S}}) \prod_{i \in S \setminus C} 1 \partial g_{\phi_{D \setminus S}}\notag \\
    &= \sum_{C \subseteq S} (-1)^{|C|} \int_{g_{\phi_{D \setminus S}} = -\infty}^{\infty} e_{\phi_{D \setminus S}}(g_{\phi_{D \setminus S}}) F_{\phi_{D \setminus S}}(g_{\phi_{D \setminus S}}) F_{\phi_C}(g_{\phi_{D \setminus S}}) \partial g_{\phi_{D \setminus S}}\notag \\
    &= \sum_{C \subseteq S} (-1)^{|C|} \int_{g_{\phi_{D \setminus S}} = -\infty}^{\infty} e_{\phi_{D \setminus S}}(g_{\phi_{D \setminus S}}) F_{\phi_{(D \setminus S) \cup C}}(g_{\phi_{D \setminus S}}) \partial g_{\phi_{D \setminus S}}\notag \\
    &= \sum_{C \subseteq S} (-1)^{|C|} (1 - 0) \frac{\exp(\phi_{D \setminus S})}{\exp(\phi_{(D \setminus S) \cup C})}\notag \\
    &= \sum_{C \subseteq S} (-1)^{|C|} \frac{1 - \sum_{i \in S} p(i)}{1 - \sum_{i \in S \setminus C} p(i)}. \label{eq:p_S_O_2_k}
\end{align}

\paragraph{Computation of $p^{D \setminus C}(S \setminus C)$.}
When using the Gumbel-Top-$k$ trick over the restricted domain $D \setminus C$, we do \emph{not} need to renormalize the log-probabilities $\phi_s, s \in D \setminus C$ since the Gumbel-Top-$k$ trick applies to unnormalized log-probabilities. Also, assuming $C \subseteq S^k$, it holds that $(D \setminus C) \setminus (S \setminus C) = D \setminus S$. This means that we can compute $p^{D \setminus C}(S \setminus C)$ similar to \eqref{eq:P_S_indefinite_integral}:
\begin{align}
    p^{D \setminus C}(S \setminus C) &= P(\min_{i \in S \setminus C} G_{\phi_i} > G_{\phi_{(D \setminus C) \setminus (S \setminus C)}}) \notag \\
    &= P(\min_{i \in S \setminus C} G_{\phi_i} > G_{\phi_{D \setminus S}}) \notag \\
    &= \int_{g_{\phi_{D \setminus S}} = -\infty}^{\infty} f_{\phi_{D \setminus S}}(g_{\phi_{D \setminus S}})  \prod_{i \in S \setminus C} \left(1 - F_{\phi_i}(g_{\phi_{D \setminus S}})\right) \partial g_{\phi_{D \setminus S}}. \label{eq:P_S_without_C_integral}
\end{align}

\paragraph{Computation of $R(S^k, s)$.}
Note that, using \eqref{eq:P_b_1_cond_S_k}, it holds that
\begin{equation*}
    \sum_{s \in S^k} \frac{p^{D \setminus \{s\}}(S^k \setminus \{s\}) p(s)}{p(S^k)} = \sum_{s \in S^k} P(b_1 = s|S^k) = 1
\end{equation*}
from which it follows that
\begin{equation*}
\label{eq:P_S_recursive}
    p(S^k) = \sum_{s \in S^k} p^{D \setminus \{s\}}(S^k \setminus \{s\}) p(s)
\end{equation*}
such that 
\begin{equation}
\label{eq:computation_leave_one_out}
    R(S^k, s) = \frac{p^{D \setminus \{s\}}(S^k \setminus \{s\})}{p(S^k)} = \frac{p^{D \setminus \{s\}}(S^k \setminus \{s\})}{\sum_{s' \in S^k} p^{D \setminus \{s'\}}(S^k \setminus \{s'\}) p(s')}.
\end{equation}
This means that, to compute the leave-one-out ratio for all $s \in S^k$, we only need to compute $p^{D \setminus \{s\}}(S^k \setminus \{s\})$ for $s \in S^k$. When using the numerical integration or summation in $O(2^k)$, we can reuse computation, whereas using the naive method, the cost is $O(k \cdot (k-1)!) = O(k!)$, making the total computational cost comparable to computing just $p(S^k)$, and the same holds when computing the `second-order' leave one out ratios for the built-in baseline (\eqref{eq:unord_set_pg_bl_estimator}).

\clearpage
\paragraph{Details of numerical integration.}
For computation of the leave-one-out ratio (\eqref{eq:computation_leave_one_out}) for large $k$ we can use the numerical integration, where we need to compute \eqref{eq:P_S_without_C_integral} with $C = \{s\}$. For this purpose, we rewrite the integral as
\begin{align*}
    p^{D \setminus C}(S \setminus C) &= \int_{g_{\phi_{D \setminus S}} = -\infty}^{\infty} f_{\phi_{D \setminus S}}(g_{\phi_{D \setminus S}})  \prod_{i \in S \setminus C} \left(1 - F_{\phi_i}(g_{\phi_{D \setminus S}})\right) \partial g_{\phi_{D \setminus S}} \\
    &= \int_{u = 0}^{1} \prod_{i \in S \setminus C} \left(1 - F_{\phi_i}\left(F^{-1}_{\phi_{D \setminus S}}(u)\right)\right) \partial u \\
    &= \int_{u = 0}^{1} \prod_{i \in S \setminus C} \left(1 - u^{\exp(\phi_i - \phi_{D \setminus S})}\right) \partial u \\
    &= \exp(b) \cdot \int_{v = 0}^{1} v^{\exp(b) - 1} \prod_{i \in S \setminus C} \left(1 - v^{\exp(\phi_i - \phi_{D \setminus S} + b)}\right) \partial v \\
    &= \exp(a + \phi_{D \setminus S}) \cdot \int_{v = 0}^{1} v^{\exp(a + \phi_{D \setminus S}) - 1} \prod_{i \in S \setminus C} \left(1 - v^{\exp(\phi_i + a)}\right) \partial v.
\end{align*}
Here we have used change of variables $v = u^{exp(-b)}$ and $a = b - \phi_{D \setminus S}$. This form allows to compute the integrands efficiently, as
\begin{equation*}
    \prod_{i \in S \setminus C} \left(1 - v^{\exp(\phi_i + a)}\right) = \frac{\prod_{i \in S} \left(1 - v^{\exp(\phi_i + a)}\right)}{\prod_{i \in C} \left(1 - v^{\exp(\phi_i + a)}\right)}
\end{equation*}
where the numerator only needs to computed once, and, since $C = \{s\}$ when computing 
\eqref{eq:computation_leave_one_out}, the denominator only consists of a single term.

The choice of $a$ may depend on the setting, but we found that $a = 5$ is a good default option which leads to an integral that is generally smooth and can be accurately approximated using the trapezoid rule. We compute the integrands in logarithmic space and sum the terms using the stable \textsc{logsumexp} trick. In our code we provide an implementation which also computes all second-order leave-one-out ratios efficiently. 

\section{The sum-and-sample estimator}
\subsection{Unbiasedness of the sum-and-sample estimator}
\label{app:proof_sas_unbiased}
We show that the sum-and-sample estimator is unbiased for any set $C \subset D$ (see also \citet{liang2018memory,liu2019rao}):
\begin{align*}
    &\hphantom{=} \mathbb{E}_{x \sim p^{D \setminus C}(x)}\left[ \sum_{c \in C} p(c) f(c) + \left(1 - \sum_{x \in C} p(c) \right) f(x)\right] \\
    &= \sum_{c \in C} p(c) f(c) + \left(1 - \sum_{c \in C} p(c) \right) \mathbb{E}_{x \sim p^{D \setminus C}(x)}[f(x)] \\
    &= \sum_{c \in C} p(c) f(c) + \left(1 - \sum_{c \in C} p(c) \right) \sum_{x \in D \setminus C} \frac{p(x)}{1 - \sum_{c \in C} p(c)} f(x) \\
    &= \sum_{c \in C} p(c) f(c) + \sum_{x \in D \setminus C} p(x) f(x) \\
    &= \sum_{x \in D} p(x) f(x) \\
    &= \mathbb{E}_{x \sim p(x)}[f(x)]
\end{align*}

\subsection{Rao-Blackwellization of the stochastic sum-and-sample estimator}
\label{app:proof_sas_rao}
In this section we give the proof that Rao-Blackwellizing the stochastic sum-and-sample estimator results in the unordered set estimator.

\begin{theorem}
\label{thm:sas_rao}
Rao-Blackwellizing the stochastic sum-and-sample estimator results in the unordered set estimator, i.e.
\begin{equation}
\label{eq:sas_rao_app}
    \mathbb{E}_{B^k \sim p(B^k|S^k)} \left[ \sum_{j = 1}^{k-1} p(b_j) f(b_j) + \left(1 - \sum_{j=1}^{k-1} p(b_j)\right) f(b_k) \right]
    = \sum_{s \in S^k} p(s) R(S^k, s) f(s).
\end{equation}
\end{theorem}

\begin{proof}
To give the proof, we first prove three Lemmas.

\begin{lemma}
\label{lem:P_b_k_cond_S}
\begin{equation}
    P(b_k = s|S^k) = \frac{p(S^k \setminus \{s\})}{p(S^k)}\frac{p(s)}{1 - \sum_{s' \in S^k \setminus \{s\}} p(s')}
\end{equation}
\end{lemma}
\begin{proof}
Similar to the derivation of $P(b_1 = s|S^k)$ (\eqref{eq:P_b_1_cond_S_k} in the main paper), we can write:
\begin{align*}
    P(b_k = s|S^k) &= \frac{P(S^k \cap b_k = s)}{p(S^k)} \\
    &= \frac{p(S^k \setminus \{s\}) p^{D \setminus (S^k \setminus \{s\})}(s)}{p(S^k)} \\
    &= \frac{p(S^k \setminus \{s\})}{p(S^k)}\frac{p(s)}{1 - \sum_{s' \in S^k \setminus \{s\}} p(s')}.
\end{align*}
The step from the first to the second row comes from analyzing the event $S^k \cap b_k = s$ using sequential sampling: to sample $S^k$ (including $s$) with $s$ being the $k$-th element means that we should first sample $S^k \setminus \{s\}$ (in any order), and then sample $s$ from the distribution restricted to $D \setminus (S^k \setminus \{s\})$.
\end{proof}

\begin{lemma}
\label{lem:P_S_relation_without_s}
\begin{equation}
\label{eq:P_S_relation_without_s}
    p(S) + p(S \setminus \{s\})\frac{1 - \sum_{s' \in S} p(s')}{1 - \sum_{s' \in S \setminus \{s\}} p(s')} = p^{D \setminus \{s\}}(S \setminus \{s\})
\end{equation}
\end{lemma}
Dividing \eqref{eq:p_S_O_2_k} by $1 - \sum_{s' \in S} p(s')$ on both sides, we obtain
\begin{proof}
\begin{align*}
    &\frac{p(S)}{1 - \sum_{s' \in S} p(s')}  \\
    =& \sum_{C \subseteq S} (-1)^{|C|} \frac{1}{1 - \sum_{s' \in S \setminus C} p(s')} \\
    =& \sum_{C \subseteq S \setminus \{s\}} \left( (-1)^{|C|} \frac{1}{1 - \sum_{s' \in S \setminus C} p(s')} + (-1)^{|C \cup \{s\}|} \frac{1}{1 - \sum_{s' \in S \setminus (C \cup \{s\})} p(s')} \right) \\
    =& \sum_{C \subseteq S \setminus \{s\}} (-1)^{|C|} \frac{1}{1 - \sum_{s' \in S \setminus C} p(s')} + \sum_{C \subseteq S \setminus \{s\}} (-1)^{|C \cup \{s\}|} \frac{1}{1 - \sum_{s' \in S \setminus (C \cup \{s\})} p(s')} \\
    =& \sum_{C \subseteq S \setminus \{s\}} (-1)^{|C|} \frac{1}{1 - p(s) - \sum_{s' \in (S \setminus \{s\}) \setminus C} p(s')} - \sum_{C \subseteq S \setminus \{s\}} (-1)^{|C|} \frac{1}{1 - \sum_{s' \in (S \setminus \{s\}) \setminus C} p(s')} \\
    =& \frac{1}{1 - p(s)} \sum_{C \subseteq S \setminus \{s\}} (-1)^{|C|} \frac{1}{1 - \sum_{s' \in (S \setminus \{s\}) \setminus C} \frac{p(s')}{1 - p(s)}} - \frac{p(S \setminus \{s\})}{1 - \sum_{s' \in S \setminus \{s\}} p(s')} \\
    =& \frac{1}{1 - p(s)} \frac{p^{D \setminus \{s\}}(S \setminus \{s\})}{1 - \sum_{s' \in S \setminus \{s\}} \frac{p(s')}{1 - p(s)}} - \frac{p(S \setminus \{s\})}{1 - \sum_{s' \in S \setminus \{s\}} p(s')} \\
    =& \frac{p^{D \setminus \{s\}}(S \setminus \{s\})}{1 - p(s) - \sum_{s' \in S \setminus \{s\}} p(s')} - \frac{p(S \setminus \{s\})}{1 - \sum_{s' \in S \setminus \{s\}} p(s')} \\
    =& \frac{p^{D \setminus \{s\}}(S \setminus \{s\})}{1 - \sum_{s' \in S} p(s')} - \frac{p(S \setminus \{s\})}{1 - \sum_{s' \in S \setminus \{s\}} p(s')}. \\
\end{align*}
Multiplying by $1 - \sum_{s' \in S} p(s')$ and rearranging terms proves Lemma \ref{lem:P_S_relation_without_s}.
\end{proof}

\begin{lemma}
\label{lem:R_S_s_from_last_sample}
\begin{equation}
    p(s) + \left(1 - \sum_{s' \in S^k} p(s')\right) P(b_k = s|S^k) = p(s) R(S^k, s)
\end{equation}
\end{lemma}
\begin{proof}
First using Lemma \ref{lem:P_b_k_cond_S} and then Lemma \ref{lem:P_S_relation_without_s} we find
\begin{align*}
    &p(s) + \left(1 - \sum_{s' \in S^k} p(s')\right) P(b_k = s|S^k) \\
    =&p(s) + \left(1 - \sum_{s' \in S^k} p(s')\right) \frac{p(S^k \setminus \{s\})}{p(S^k)}\frac{p(s)}{1 - \sum_{s' \in S^k \setminus \{s\}} p(s')} \\
    =& \frac{p(s)}{p(S^k)}\left(p(S^k) + \frac{1 - \sum_{s' \in S^k} p(s')}{1 - \sum_{s' \in S^k \setminus \{s\}} p(s')} p(S^k \setminus \{s\}) \right) \\
    =& \frac{p(s)}{{p(S^k)}} p^{D \setminus \{s\}}(S^k \setminus \{s\}) \\
    =& p(s) R(S^k, s).
\end{align*}
\end{proof}

Now we can complete the proof of Theorem \ref{thm:sas_rao} by adding $p(b_k)f(b_k) - p(b_k)f(b_k) = 0$ to the estimator, moving the terms independent of $B^k$ outside the expectation and using Lemma \ref{lem:R_S_s_from_last_sample}:
\begin{align*}
    &\mathbb{E}_{B^k \sim p(B^k|S^k)} \left[ \sum_{j = 1}^{k-1} p(b_j) f(b_j) + \left(1 - \sum_{j=1}^{k-1} p(b_j)\right) f(b_k) \right] \\
    =&\mathbb{E}_{B^k \sim p(B^k|S^k)} \left[ \sum_{j = 1}^k p(b_j) f(b_j) + \left(1 - \sum_{j=1}^k p(b_j)\right) f(b_k) \right] \\
    =& \sum_{s \in S^k} p(s) f(s) + \mathbb{E}_{B^k \sim p(B^k|S^k)}\left[\left(1 - \sum_{s' \in S^k} p(s')\right) f(b_k)\right] \\
    =&  \sum_{s \in S^k} p(s) f(s) + \sum_{s \in S^k} \left(1 - \sum_{s' \in S^k} p(s')\right) P(b_k = s|S^k) f(s) \\
    &= \sum_{s \in S^k} \left(p(s) + \left(1 - \sum_{s' \in S^k} p(s')\right) P(b_k = s|S^k)\right) f(s) \\
    &= \sum_{s \in S^k} p(s) R(S^k, s) f(s).
\end{align*}
\end{proof}

\subsection{The stochastic sum-and-sample estimator with multiple samples}
\label{app:proof_sas_multi}
As was discussed in \citet{liu2019rao}, one can trade off the number of summed terms and number of sampled terms to maximize the achieved variance reduction. As a generalization of Theorem \ref{thm:sas_rao} (the stochastic sum-and-sample estimator with $k - 1$ summed terms), we introduce here the stochastic sum-and-sample estimator that sums $k - m$ terms and samples $m > 1$ terms \emph{without replacement}. To estimate the sampled term, we use the unordered set estimator on the $m$ samples without replacement, on the domain restricted to $D \setminus B^{k-m}$. In general, we denote the \emph{unordered set estimator} restricted to the domain $D \setminus C$ by
\begin{equation}
\label{eq:unord_set_estimator_restricted}
    e^{\text{US}, D \setminus C}(S^k)
    = \sum_{s \in S^k \setminus C} p(s) R^{D \setminus C}(S^k, s) f(s)
\end{equation}
where $R^{D \setminus C}(S^k, s)$ is the \emph{leave-one-out ratio} restricted to the domain $D \setminus C$, similar to the second order leave-one-out ratio in \eqref{eq:second_order_leave_one_out_ratio}:
\begin{equation}
    R^{D \setminus C}(S^k, s) = \frac{p_{\bm{\theta}}^{(D \setminus C) \setminus \{s\}}((S^k \setminus C) \setminus \{s\})}{p_{\bm{\theta}}^{D \setminus C}(S^k \setminus C)}.
\end{equation}
While we can also constrain $S^k \subseteq (D \setminus C)$, this definition is consistent with \eqref{eq:second_order_leave_one_out_ratio} and allows simplified notation.

\begin{theorem}
\label{thm:sas_multi_rao}
Rao-Blackwellizing the stochastic sum-and-sample estimator with $m > 1$ samples results in the unordered set estimator, i.e.
\begin{equation}
\label{eq:def_sas_multi}
    \mathbb{E}_{B^k \sim p(B^k|S^k)} \left[ \sum_{j = 1}^{k-m} p(b_j) f(b_j) + \left(1 - \sum_{j=1}^{k-m} p(b_j)\right) e^{\text{US}, D \setminus B^{k-m}}(S^k) \right]
    = \sum_{s \in S^k} p(s) R(S^k, s) f(s).
\end{equation}
\end{theorem}

\begin{proof}
Recall that for the unordered set estimator, it holds that
\begin{equation}
    e^{\text{US}}(S^k) = \mathbb{E}_{b_1 \sim p(b_1|S^k)} \left[ f(b_1) \right] = \mathbb{E}_{x \sim p(x)} \left[ f(x) \middle| x \in S^k \right]
\end{equation}
\end{proof}
which for the restricted equivalent (with restricted distribution $p^{D \setminus C}$) translates into
\begin{equation}
    e^{\text{US}, D \setminus C}(S^k) = \mathbb{E}_{x \sim p^{D \setminus C}(x)} \left[ f(x) \middle| x \in S^k \right] = \mathbb{E}_{x \sim p(x)} \left[ f(x) \middle| x \in S^k, x \not\in C \right].
\end{equation}
Now we consider the distribution $b_{k-m+1}|S^k,B^{k-m}$: the distribution of the first element sampled (without replacement) after sampling $B^{k-m}$, given (conditionally on the event) that the set of $k$ samples is $S^k$, so we have $b_{k-m+1} \in S^k$ and $b_{k-m+1} \not\in B^{k-m}$. This means that its conditional expectation of $f(b_{k-m+1})$ is the restricted unordered set estimator for $C = B^{k-m}$ since
\begin{align}
    e^{\text{US}, D \setminus B^{k-m}}(S^k) &= \mathbb{E}_{x \sim p(x)} \left[ f(x) \middle| x \in S^k, x\not\in B^{k-m} \right] \notag \\ 
    &= \mathbb{E}_{b_{k-m+1} \sim p(b_{k-m+1}|S^k,B^{k-m})} \left[ f(b_{k-m+1}) \right] \label{eq:cond_exp_us_restr}.
\end{align}
Observing that the definition (\eqref{eq:def_sas_multi}) of the stochastic sum-and-sample estimator does not depend on the actual order of the $m$ samples, and using \eqref{eq:cond_exp_us_restr}, we can reduce the multi-sample estimator to the stochastic sum-and-sample estimator with $k' = k - m + 1$, such that the result follows from \eqref{eq:sas_rao_app}.
\begin{align}
    &\mathbb{E}_{B^{k} \sim p(B^{k}|S^k)} \left[ \sum_{j = 1}^{k-m} p(b_j) f(b_j) + \left(1 - \sum_{j=1}^{k-m} p(b_j)\right) e^{\text{US}, D \setminus B^{k-m}}(S^k) \right] \notag \\
    =& \mathbb{E}_{B^{k-m} \sim p(B^{k-m}|S^k)} \left[ \sum_{j = 1}^{k-m} p(b_j) f(b_j) + \left(1 - \sum_{j=1}^{k-m} p(b_j)\right) e^{\text{US}, D \setminus B^{k-m}}(S^k) \right] \notag \\
    =& \mathbb{E}_{B^{k-m} \sim p(B^{k-m}|S^k)} \left[ \sum_{j = 1}^{k-m} p(b_j) f(b_j) + \left(1 - \sum_{j=1}^{k-m} p(b_j)\right) \mathbb{E}_{b_{k-m+1} \sim p(b_{k-m+1}|S^k,B^{k-m})} \left[ f(b_{k-m+1}) \right] \right] \notag \\
    =& \mathbb{E}_{B^{k-m+1} \sim p(B^{k-m+1}|S^k)} \left[ \sum_{j = 1}^{k-m} p(b_j) f(b_j) + \left(1 - \sum_{j=1}^{k-m} p(b_j)\right) f(b_{k-m+1}) \right] \notag \\
    =& \mathbb{E}_{S^{k-m+1}|S^k} \left[ \mathbb{E}_{B^{k-m+1} \sim p(B^{k-m+1}|S^{k-m+1})} \left[ \sum_{j = 1}^{k-m} p(b_j) f(b_j) + \left(1 - \sum_{j=1}^{k-m} p(b_j)\right) f(b_{k-m+1}) \right] \right] \notag \\
    =& \mathbb{E}_{S^{k-m+1}|S^k} \left[ \sum_{s \in S^k} p(s) R(S^k, s) f(s) \right] \notag \\
    =& \sum_{s \in S^k} p(s) R(S^k, s) f(s).
\end{align}

\section{The importance-weighted estimator}

\subsection{Rao-Blackwellization of the importance-weighted estimator}
\label{app:proof_iw_rao}
In this section we give the proof that Rao-Blackwellizing the importance-weighted estimator results in the unordered set estimator.

\begin{theorem}
\label{thm:iw_rao}
Rao-Blackwellizing the importance-weighted estimator results in the unordered set estimator, i.e.:
\begin{equation}
    \mathbb{E}_{\kappa \sim p(\kappa|S^k)} \left[ \sum_{s \in S^k} \frac{p(s)}{1 - F_{\phi_s}(\kappa)} f(s) \right] = \sum_{s \in S^k} p(s) R(S^k, s) f(s).
\end{equation}
\end{theorem}
Here we have slightly rewritten the definition of the importance-weighted estimator, using that $q(s, a) = P(g_{\phi_s} > a) = 1 - F_{\phi_s}(a)$, where $F_{\phi_s}$ is the CDF of the Gumbel distribution (see Appendix \ref{app:notation}).

\begin{proof}
We first prove the following Lemma:
\begin{lemma}
\label{lem:E_iw_cond_S}
\begin{equation}
    \mathbb{E}_{\kappa \sim p(\kappa|S^k)} \left[ \frac{1}{1 - F_{\phi_s}(\kappa)} \right] = R(S^k, s)
\end{equation}
\end{lemma}
\begin{proof}
Conditioning on $S^k$, we know that the elements in $S^k$ have the $k$ largest perturbed log-probabilities, so $\kappa$, the $(k+1)$-th largest perturbed log-probability is the largest perturbed log-probability in $D \setminus S^k$, and satisfies $\kappa = \max_{s \in D \setminus S^k} g_{\phi_s} = g_{\phi_{D\setminus S^k}} \sim \text{Gumbel}(\phi_{D\setminus S^k})$.
Computing $p(\kappa|S^k)$ using Bayes' Theorem, we have
\begin{equation}
    p(\kappa|S^k) = \frac{p(S^k|\kappa) p(\kappa)}{p(S^k)} = \frac{\prod_{s \in S^k} (1 - F_{\phi_s}(\kappa)) f_{\phi_{D\setminus S^k}}(\kappa)}{p(S^k)}
\end{equation}

which allows us to compute (using \eqref{eq:P_S_without_C_integral} with $C = \{s\}$ and $g_{\phi_{D \setminus S}} = \kappa$)
\begin{align*}
    &\mathbb{E}_{\kappa \sim p(\kappa|S^k)} \left[ \frac{1}{1 - F_{\phi_s}(\kappa)} \right] \\
    =& \int_{\kappa = -\infty}^{\infty} p(\kappa|S^k) \frac{1}{1 - F_{\phi_s}(\kappa)} \partial \kappa \\
    =& \int_{\kappa = -\infty}^{\infty} \frac{\prod_{s \in S^k} (1 - F_{\phi_s}(\kappa)) f_{\phi_{D\setminus S^k}}(\kappa)}{p(S^k)} \frac{1}{1 - F_{\phi_s}(\kappa)} \partial \kappa \\
    =& \frac{1}{p(S^k)} \int_{\kappa = -\infty}^{\infty} \prod_{s \in S^k \setminus \{s\}} (1 - F_{\phi_s}(\kappa)) f_{\phi_{D\setminus S^k}}(\kappa) \partial \kappa \\
    =& \frac{1}{p(S^k)} p^{D \setminus \{s\}}(S \setminus \{s\}) \\
    =& R(S^k, s).
\end{align*}
\end{proof}

Using Lemma \ref{lem:E_iw_cond_S} we find 
\begin{align*}
    &\mathbb{E}_{\kappa \sim p(\kappa|S^k)} \left[ \sum_{s \in S^k} \frac{p(s)}{1 - F_{\phi_s}(\kappa)} f(s) \right] \\
    =& \sum_{s \in S^k} p(s) \mathbb{E}_{\kappa \sim p(\kappa|S^k)} \left[ \frac{1}{1 - F_{\phi_s}(\kappa)} \right] f(s) \\
    =& \sum_{s \in S^k} p(s) R(S^k, s) f(s).
\end{align*}
\end{proof}

\subsection{The importance-weighted policy gradient estimator with built-in baseline}
\label{app:importance_weighted_pg}
For self-containment we include this section, which is adapted from our unpublished workshop paper \citep{kool2019buy}.
The importance-weighted policy gradient estimator combines REINFORCE \citep{williams1992simple} with the importance-weighted estimator \citep{duffield2007priority, vieira2017estimating} in \eqref{eq:iw_estimator} which results in an unbiased estimator of the policy gradient $\nabla_{\bm{\theta}} \mathbb{E}_{p_{\bm{\theta}}(x)}[f_{\bm{\theta}}(x)]$:
\begin{equation}
\label{eq:iwpg_estimator}
    e^{\text{IWPG}}(S^k, \kappa) = \sum_{s \in S^k} \frac{p_{\bm{\theta}}(s)}{q_{\bm{\theta},\kappa}(s)} \nabla_{\bm{\theta}} \log p_{\bm{\theta}}(s) f(s) = \sum_{s \in S^k} \frac{\nabla_{\bm{\theta}} p_{\bm{\theta}}(s)}{q_{\bm{\theta},\kappa}(s)} f(s)
\end{equation}
Recall that $\kappa$ is the $(k+1)$-th largest perturbed log-probability (see Section \ref{sec:rao_bw_other}). 
We compute a lower variance but biased variant by normalizing the importance weights using the normalization $W(S^k) = \sum_{s \in S^k} \frac{p_{\bm{\theta}}(s)}{q_{\bm{\theta},\kappa}(s)}$.

As we show in \citet{kool2019buy}, we can include a `baseline'  $B(S^k) = \sum_{s \in S^k} \frac{p_{\bm{\theta}}(s)}{q_{\bm{\theta},\kappa}(s)} f(s)$ and correct for the bias (since it depends on the complete sample $S^k$) by weighting individual terms of the estimator  by $1 - p_{\bm{\theta}}(s) + \frac{p_{\bm{\theta}}(s)}{q_{\bm{\theta},\kappa}(s)}$:
\begin{equation}
\label{eq:iwpgbl_estimator}
    e^{\text{IWPGBL}}(S^k, \kappa) =  \sum_{s \in S^k} \frac{\nabla_{\bm{\theta}} p_{\bm{\theta}}(s)}{q_{\bm{\theta},\kappa}(s)} \left( f(s) \left(1 - p_{\bm{\theta}}(s) + \frac{ p_{\bm{\theta}}(s)}{q_{\bm{\theta},\kappa}(s)} \right) - B(S^k) \right)
\end{equation}

For the normalized version, we use the normalization $W(S^k) = \sum_{s \in S^k} \frac{p_{\bm{\theta}}(s)}{q_{\bm{\theta},\kappa}(s)}$ for the baseline, and $W_i(S^k) = W(S^k) - \frac{p_{\bm{\theta}}(s)}{q_{\bm{\theta},\kappa}(s)} + p_{\bm{\theta}}(s)$ to normalize the individual terms:
\begin{equation}
\label{eq:niwpgbl_estimator}
    \nabla_{\bm{\theta}} \mathbb{E}_{y\sim p_{\bm{\theta}}(y)} \left[ f(y) \right] \approx \sum_{s \in S^k} \frac{1}{W_i(S^k)} \cdot \frac{\nabla_{\bm{\theta}} p_{\bm{\theta}}(s)}{q_{\bm{\theta},\kappa}(s)} \left( f(s) - \frac{B(S^k)}{W(S^k)} \right)
\end{equation}
It seems odd to normalize the terms in the outer sum by $\frac{1}{W_i(S^k)}$ instead of $\frac{1}{W(S^k)}$, but \eqref{eq:niwpgbl_estimator} can be rewritten into a form similar to  \eqref{eq:unord_set_pg_bl_estimator}, i.e.\ with a different baseline for each sample, but this form is more convenient for implementation \citep{kool2019buy}.

\section{The unordered set policy gradient estimator}

\subsection{Proof of unbiasedness of the unordered set policy gradient estimator with baseline}
\label{app:proof_us_bl_unbiased}
To prove the unbiasedness of result we need to prove that the control variate has expectation $0$:
\begin{lemma}
\begin{equation}
    \mathbb{E}_{S^k \sim p_{\bm{\theta}}(S^k)} \left[ \sum_{s \in S^k} \nabla_{\bm{\theta}} p_{\bm{\theta}}(s) R(S^k, s) \sum_{s' \in S^k} p_{\bm{\theta}}(s') R^{D \setminus \{s\}}(S^k, s') f(s') \right] = 0.
\end{equation}
\end{lemma}
\begin{proof}

Similar to \eqref{eq:P_b_1_cond_S_k}, we apply Bayes' Theorem conditionally on $b_1 = s$ to derive for $s' \neq s$
\begin{align}
    P(b_2 = s'|S^k, b_1 = s) &= \frac{P(S^k|b_2 = s', b_1 = s)P(b_2 = s' | b_1 = s')}{P(S^k | b_1 = s)} \notag \\
    &= \frac{p_{\bm{\theta}}^{D \setminus \{s, s'\}}(S^k \setminus \{s, s'\}) p_{\bm{\theta}}^{D \setminus \{s\}}(s')}{p_{\bm{\theta}}^{D \setminus \{s\}}(S^k \setminus \{s\})} \notag \\
    &= \frac{p_{\bm{\theta}}(s')}{1 - p_{\bm{\theta}}(s)} R^{D \setminus \{s\}}(S^k, s'). \label{eq:P_b_2_cond_S_k_b_1}
\end{align}
For $s' = s$ we have $R^{D \setminus \{s\}}(S^k, s') = 1$ by definition, so using \eqref{eq:P_b_2_cond_S_k_b_1} we can show that
\begin{align*}
    &\hphantom{=} \sum_{s' \in S^k} p_{\bm{\theta}}(s') R^{D \setminus \{s\}}(S^k, s') f(s') \\
    &= p_{\bm{\theta}}(s)f(s) + \sum_{s' \in S^k \setminus \{s\}} p_{\bm{\theta}}(s') R^{D \setminus \{s\}}(S^k, s') f(s') \\
    &= p_{\bm{\theta}}(s)f(s) + (1 - p_{\bm{\theta}}(s)) \sum_{s' \in S^k \setminus \{s\}} \frac{p_{\bm{\theta}}(s')}{1 - p_{\bm{\theta}}(s)} R^{D \setminus \{s\}}(S^k, s') f(s') \\
    &= p_{\bm{\theta}}(s)f(s) + (1 - p_{\bm{\theta}}(s)) \sum_{s' \in S^k \setminus \{s\}} P(b_2 = s' | S^k, b_1 = s) f(s') \\
    &= p_{\bm{\theta}}(s)f(s) + (1 - p_{\bm{\theta}}(s)) \mathbb{E}_{b_2 \sim p_{\bm{\theta}}(b_2|S^k,b_1 = s)} \left[ f(b_2) \right] \\
    &= \mathbb{E}_{b_2 \sim p_{\bm{\theta}}(b_2|S^k,b_1 = s)} \left[ p_{\bm{\theta}}(b_1)f(b_1) + (1 - p_{\bm{\theta}}(b_1))  f(b_2) \right].
\end{align*}

Now we can show that the control variate is actually the result of Rao-Blackwellization:
\begin{align*}
    &\hphantom{=} \mathbb{E}_{S^k \sim p_{\bm{\theta}}(S^k)} \left[ \sum_{s \in S^k} \nabla_{\bm{\theta}} p_{\bm{\theta}}(s) R(S^k, s) \sum_{s' \in S^k} p_{\bm{\theta}}(s') R^{D \setminus \{s\}}(S^k, s') f(s') \right] \\
    &= \mathbb{E}_{S^k \sim p_{\bm{\theta}}(S^k)} \left[ \sum_{s \in S^k} p_{\bm{\theta}}(s) R(S^k, s) \nabla_{\bm{\theta}} \log p_{\bm{\theta}}(s) \sum_{s' \in S^k} p_{\bm{\theta}}(s') R^{D \setminus \{s\}}(S^k, s') f(s') \right] \\
    &= \mathbb{E}_{S^k \sim p_{\bm{\theta}}(S^k)} \left[ \sum_{s \in S^k} P(b_1 = s | S^k) \nabla_{\bm{\theta}} \log p_{\bm{\theta}}(s) \sum_{s' \in S^k} p_{\bm{\theta}}(s') R^{D \setminus \{s\}}(S^k, s') f(s') \right] \\
    &= \mathbb{E}_{S^k \sim p_{\bm{\theta}}(S^k)} \left[ \mathbb{E}_{b_1 \sim p_{\bm{\theta}}(b_1|S^k)} \left[ \nabla_{\bm{\theta}} \log p_{\bm{\theta}}(b_1) \sum_{s' \in S^k} p_{\bm{\theta}}(s') R^{D \setminus \{b_1\}}(S^k, s') f(s') \right] \right] \\
    &= \mathbb{E}_{S^k \sim p_{\bm{\theta}}(S^k)} \left[ \mathbb{E}_{b_1 \sim p_{\bm{\theta}}(b_1|S^k)} \left[ \nabla_{\bm{\theta}} \log p_{\bm{\theta}}(b_1) \mathbb{E}_{b_2 \sim p_{\bm{\theta}}(b_2|S^k,b_1)} \left[ p_{\bm{\theta}}(b_1)f(b_1) + (1 - p_{\bm{\theta}}(b_1)) f(b_2) \right] \right] \right] \\
    &= \mathbb{E}_{S^k \sim p_{\bm{\theta}}(S^k)} \left[ \mathbb{E}_{B^k \sim p_{\bm{\theta}}(B^k|S^k)} \left[ \nabla_{\bm{\theta}} \log p_{\bm{\theta}}(b_1) \left( p_{\bm{\theta}}(b_1)f(b_1) + (1 - p_{\bm{\theta}}(b_1)) f(b_2) \right) \right] \right] \\
    &= \mathbb{E}_{B^k \sim p_{\bm{\theta}}(B^k)} \left[ \nabla_{\bm{\theta}} \log p_{\bm{\theta}}(b_1) \left( p_{\bm{\theta}}(b_1)f(b_1) + (1 - p_{\bm{\theta}}(b_1)) f(b_2) \right) \right]
\end{align*}

This expression depends only on $b_1$ and $b_2$ and we recognize the stochastic sum-and-sample estimator for $k = 2$ used as `baseline'. As a special case of \eqref{eq:sas_unbiased} for $C = \{b_1\}$, we have
\begin{equation} 
\mathbb{E}_{b_2 \sim p_{\bm{\theta}}(b_2|b_1)} \left[ \left( p_{\bm{\theta}}(b_1)f(b_1) + (1 - p_{\bm{\theta}}(b_1)) f(b_2) \right) \right] = \mathbb{E}_{i \sim p_{\bm{\theta}}(i)} \left[ f(i) \right].
\end{equation}
Using this, and the fact that $\mathbb{E}_{b_1 \sim p_{\bm{\theta}}(b_1)} \left[ \nabla_{\bm{\theta}} \log p_{\bm{\theta}}(b_1) \right] = \nabla_{\bm{\theta}} \mathbb{E}_{b_1 \sim p_{\bm{\theta}}(b_1)} \left[ 1 \right] = \nabla_{\bm{\theta}} 1 = 0$ we find

\begin{align*}
    &\hphantom{=} \mathbb{E}_{S^k \sim p_{\bm{\theta}}(S^k)} \left[ \sum_{s \in S^k} \nabla_{\bm{\theta}} p_{\bm{\theta}}(s) R(S^k, s) \sum_{s' \in S^k} p_{\bm{\theta}}(s') R^{D \setminus \{s\}}(S^k, s') f(s') \right] \\
    &= \mathbb{E}_{B^k \sim p_{\bm{\theta}}(B^k)} \left[ \nabla_{\bm{\theta}} \log p_{\bm{\theta}}(b_1) \left( p_{\bm{\theta}}(b_1)f(b_1) + (1 - p_{\bm{\theta}}(b_1)) f(b_2) \right) \right] \\
    &= \mathbb{E}_{b_1 \sim p_{\bm{\theta}}(b_1)} \left[ \nabla_{\bm{\theta}} \log p_{\bm{\theta}}(b_1) \mathbb{E}_{b_2 \sim p_{\bm{\theta}}(b_2|b_1)} \left[ \left( p_{\bm{\theta}}(b_1)f(b_1) + (1 - p_{\bm{\theta}}(b_1)) f(b_2) \right) \right] \right] \\
    &= \mathbb{E}_{b_1 \sim p_{\bm{\theta}}(b_1)} \left[ \nabla_{\bm{\theta}} \log p_{\bm{\theta}}(b_1) \mathbb{E}_{x \sim p_{\bm{\theta}}(x)} \left[ f(x) \right] \right] \\
    &= \mathbb{E}_{b_1 \sim p_{\bm{\theta}}(b_1)} \left[ \nabla_{\bm{\theta}} \log p_{\bm{\theta}}(b_1) \right] \mathbb{E}_{x \sim p_{\bm{\theta}}(x)} \left[ f(x) \right] \\
    &= 0 \cdot \mathbb{E}_{x \sim p_{\bm{\theta}}(x)} \left[ f(x) \right] \\
    &= 0
\end{align*}

\end{proof}

\clearpage
\section{The RISK estimator}

\subsection{Proof of built-in baseline}
\label{app:proof_risk_baseline}
We show that the RISK estimator, taking gradients through the normalization factor actually has a built-in baseline. We first use the log-derivative trick to rewrite the gradient of the ratio as the ratio times the logarithm of the gradient, and then swap the summation variables in the double sum that arises:
\begin{align*}
    e^{\text{RISK}}(S) &= \sum_{s \in S} \nabla_{\bm{\theta}} \left(\frac{p_{\bm{\theta}}(s)}{\sum_{s' \in S} p_{\bm{\theta}}(s')}\right) f(s)  \\
    &= \sum_{s \in S}  \frac{p_{\bm{\theta}}(s)}{\sum_{s' \in S} p_{\bm{\theta}}(s')} \nabla_{\bm{\theta}} \log \left(\frac{p_{\bm{\theta}}(s)}{\sum_{s' \in S} p_{\bm{\theta}}(s')}\right) f(s) \\
    &= \sum_{s \in S}  \frac{p_{\bm{\theta}}(s)}{\sum_{s' \in S} p_{\bm{\theta}}(s')} \left( \nabla_{\bm{\theta}} \log p_{\bm{\theta}}(s) - \nabla_{\bm{\theta}} \log \sum_{s' \in S} p_{\bm{\theta}}(s')\right) f(s)  \\
    &= \sum_{s \in S}  \frac{p_{\bm{\theta}}(s)}{\sum_{s' \in S} p_{\bm{\theta}}(s')} \left( \frac{\nabla_{\bm{\theta}} p_{\bm{\theta}}(s)}{p_{\bm{\theta}}(s)} - \frac{\sum_{s' \in S}\nabla_{\bm{\theta}} p_{\bm{\theta}}(s')}{\sum_{s' \in S} p_{\bm{\theta}}(s')}\right) f(s) \\
    &= \sum_{s \in S}  \frac{\nabla_{\bm{\theta}} p_{\bm{\theta}}(s) f(s)}{\sum_{s' \in S} p_{\bm{\theta}}(s')}  - \frac{\sum_{s, s' \in S} p_{\bm{\theta}}(s) \nabla_{\bm{\theta}}  p_{\bm{\theta}}(s') f(s)}{\left(\sum_{s' \in S} p_{\bm{\theta}}(s')\right)^2} \\
    &= \sum_{s \in S}  \frac{\nabla_{\bm{\theta}} p_{\bm{\theta}}(s) f(s)}{\sum_{s' \in S} p_{\bm{\theta}}(s')}  - \frac{\sum_{s, s' \in S} p_{\bm{\theta}}(s') \nabla_{\bm{\theta}}  p_{\bm{\theta}}(s) f(s')}{\left(\sum_{s' \in S} p_{\bm{\theta}}(s')\right)^2} \\
    &= \sum_{s \in S}  \frac{\nabla_{\bm{\theta}} p_{\bm{\theta}}(s)}{\sum_{s' \in S} p_{\bm{\theta}}(s')} \left(f(s)  - \frac{\sum_{s' \in S} p_{\bm{\theta}}(s')  f(s')}{\sum_{s' \in S} p_{\bm{\theta}}(s')} \right) \\
    &= \sum_{s \in S}  \frac{\nabla_{\bm{\theta}} p_{\bm{\theta}}(s)}{\sum_{s'' \in S} p_{\bm{\theta}}(s'')} \left(f(s) - \sum_{s' \in S} \frac{ p_{\bm{\theta}}(s')}{\sum_{s'' \in S} p_{\bm{\theta}}(s'')} f(s') \right).
\end{align*}

\section{Categorical Variational Auto-Encoder}
\subsection{Experimental details}
\label{app:categorical_vae_details}
We use the code\footnote{\url{https://github.com/ARM-gradient/ARSM}} by \citet{yin2019arsm} to reproduce their categorical VAE experiment, of which we include details here for self-containment. The dataset is MNIST, statically binarized by thresholding at $0.5$ (although we include results using the standard binarized dataset by \citet{salakhutdinov2008quantitative, larochelle2011neural} in Section \ref{app:categorical_vae_extra_results}). The latent representation $\bm{z}$ is $K = 20$ dimensional with $C = 10$ categories per dimension with a uniform prior $p(z_k = c) = 1 / C, k = 1, ..., K$. The encoder is parameterized by $\bm{\phi}$ as $q_{\bm{\phi}}(\bm{z}|\bm{x}) = \prod_{k} q_{\bm{\phi}}(z_k|\bm{x})$ and has two fully connected hidden layers with 512 and 256 hidden nodes respectively, with LeakyReLU ($\alpha = 0.1$) activations. The decoder, parameterized by $\bm{\theta}$, is given by $p_{\bm{\theta}}(\bm{x}|\bm{z}) = \prod_i p_{\bm{\theta}}(x_i|\bm{z})$, where $x_i \in \{0, 1\}$ are the pixel values, and has fully connected hidden layers with 256 and 512 nodes and LeakyReLU activation.

\paragraph{ELBO optimization.}
The evidence lower bound (ELBO) that we optimize is given by
\begin{align}
    \mathcal{L}(\bm{\phi}, \bm{\theta}) &= \mathbb{E}_{\bm{z} \sim q_{\bm{\phi}}(\bm{z}|\bm{x})} \left[ \ln p_{\bm{\theta}}(\bm{x}|\bm{z}) + \ln p(\bm{z}) - \ln q_{\bm{\phi}}(\bm{z}|\bm{x}) \right] \label{eq:elbo_sample_kl} \\
    &= \mathbb{E}_{\bm{z} \sim q_{\bm{\phi}}(\bm{z}|\bm{x})} \left[ \ln p_{\bm{\theta}}(\bm{x}|\bm{z}) \right] - KL(q_{\bm{\phi}}\left(\bm{z}|\bm{x})||p(\bm{z})\right)  \label{eq:elbo_analytic_kl}.
\end{align}
For the decoder parameters $\bm{\theta}$, since $q_{\bm{\phi}}(\bm{z}|\bm{x})$ does not depend on $\bm{\theta}$, it follows that
\begin{equation}
    \nabla_{\bm{\theta}} \mathcal{L}(\bm{\phi}, \bm{\theta}) = \mathbb{E}_{\bm{z} \sim q_{\bm{\phi}}(\bm{z}|\bm{x})} \left[ \nabla_{\bm{\theta}} \ln p_{\bm{\theta}}(\bm{x}|\bm{z}) \right].
\end{equation}
For the encoder parameters $\bm{\phi}$, we can write $\nabla_{\bm{\phi}} \mathcal{L}(\bm{\phi}, \bm{\theta})$ using \eqref{eq:elbo_analytic_kl} and \eqref{eq:reinforce_pathwise_identity} as
\begin{equation}
\label{eq:elbo_grad_analytic_kl}
    \nabla_{\bm{\phi}} \mathcal{L}(\bm{\phi}, \bm{\theta}) = \mathbb{E}_{\bm{z} \sim q_{\bm{\phi}}(\bm{z}|\bm{x})} \left[ \nabla_{\bm{\phi}} \ln q_{\bm{\phi}}(\bm{z}|\bm{x}) \ln p_{\bm{\theta}}(\bm{x}|\bm{z}) \right] - \nabla_{\bm{\phi}} KL(q_{\bm{\phi}}\left(\bm{z}|\bm{x})||p(\bm{z})\right).
\end{equation}
This assumes we can compute the KL divergence analytically. Alternatively, we can use a sample estimate for the KL divergence, and use \eqref{eq:elbo_sample_kl} with \eqref{eq:reinforce_pathwise_identity} to obtain
\begin{align}
    \nabla_{\bm{\phi}} \mathcal{L}(\bm{\phi}, \bm{\theta}) &= \mathbb{E}_{\bm{z} \sim q_{\bm{\phi}}(\bm{z}|\bm{x})} \left[ \nabla_{\bm{\phi}} \ln q_{\bm{\phi}}(\bm{z}|\bm{x}) ( \ln p_{\bm{\theta}}(\bm{x}|\bm{z}) + \ln p(\bm{z}) - \ln q_{\bm{\phi}}(\bm{z}|\bm{x}) ) + \nabla_{\bm{\phi}} \ln q_{\bm{\phi}}(\bm{z}|\bm{x}) \right] \label{eq:elbo_grad_sample_kl} \\
    &= \mathbb{E}_{\bm{z} \sim q_{\bm{\phi}}(\bm{z}|\bm{x})} \left[ \nabla_{\bm{\phi}} \ln q_{\bm{\phi}}(\bm{z}|\bm{x}) ( \ln p_{\bm{\theta}}(\bm{x}|\bm{z}) - \ln q_{\bm{\phi}}(\bm{z}|\bm{x}) ) \right] \label{eq:elbo_grad_sample_kl_minimal}.
\end{align}
Here we have left out the term $\mathbb{E}_{\bm{z} \sim q_{\bm{\phi}}(\bm{z}|\bm{x})} \left[ \nabla_{\bm{\phi}} \ln q_{\bm{\phi}}(\bm{z}|\bm{x}) \right] = 0$, similar to \citet{roeder2017sticking}, and, assuming a uniform (i.e.\ constant) prior $\ln p(\bm{z})$, the term $\mathbb{E}_{\bm{z} \sim q_{\bm{\phi}}(\bm{z}|\bm{x})} \left[ \nabla_{\bm{\phi}} \ln q_{\bm{\phi}}(\bm{z}|\bm{x}) \ln p(\bm{z}) \right] = 0$. With a built-in baseline, this second term cancels out automatically, even if it is implemented.
Despite the similarity of the \eqref{eq:elbo_sample_kl} and \eqref{eq:elbo_analytic_kl}, their gradient estimates (\eqref{eq:elbo_grad_sample_kl} and \eqref{eq:elbo_grad_analytic_kl}) are structurally dissimilar and care should be taken to implement the REINFORCE estimator (or related estimators such as ARSM and the unordered set estimator) correctly using automatic differentiation software. Using Gumbel-Softmax and RELAX, we take gradients `directly' through the objective in \eqref{eq:elbo_analytic_kl}.

We optimize the ELBO using the analytic KL for 1000 epochs using the Adam \citep{kingma2015adam} optimizer. We use a learning rate of $10^{-3}$ for all estimators except Gumbel-Softmax and RELAX, which use a learning rate of $10^{-4}$ as we found they diverged with a higher learning rate. For ARSM, as an exception we use the sample KL, and a learning rate of $3 \cdot 10^{-4}$, as suggested by the authors. All reported ELBO values are computed using the analytic KL. Our code is publicly available\footnote{\url{https://github.com/wouterkool/estimating-gradients-without-replacement}}.

\subsection{Additional results}
\label{app:categorical_vae_extra_results}
\paragraph{Gradient variance during training.}
We also evaluate gradient variance of different estimators during different stages of training. We measure the variance of different estimators with $k = 4$ samples during training with REINFORCE with replacement, such that all estimators are computed for the same model parameters. The results during training, given in Figure \ref{fig:vae_grads}, are similar to the results for the trained model in Table \ref{tab:vae_grads}, except for at the beginning of training, although the rankings of different estimator are mostly the same.

\begin{figure}[b]
    \centering
    \begin{subfigure}[b]{0.48\textwidth}
        \includegraphics[width=\textwidth]{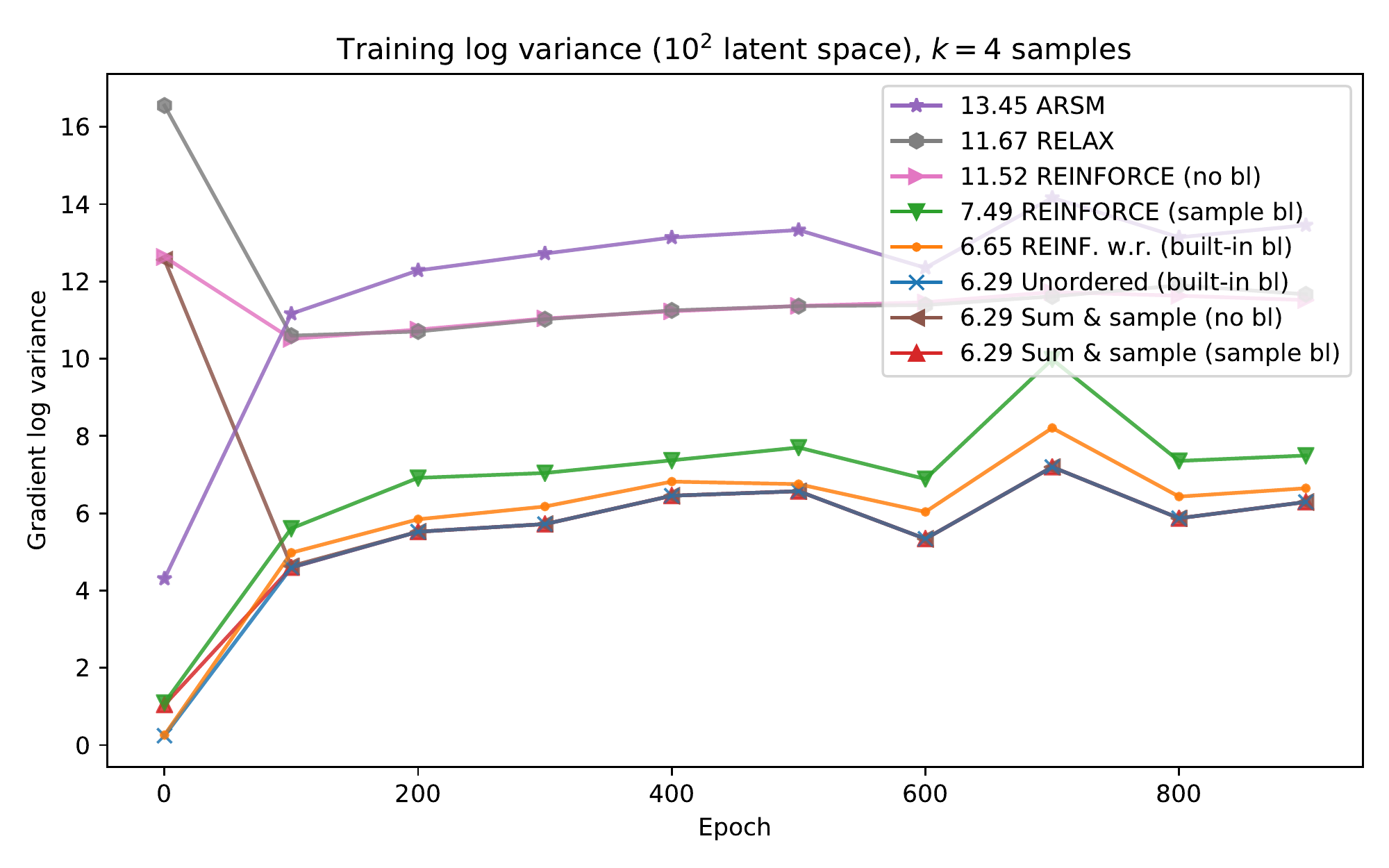}
        \vskip -0.2cm
        \caption{Small domain (latent space size $10^2$)}
    \end{subfigure}
    \begin{subfigure}[b]{0.48\textwidth}
        \includegraphics[width=\textwidth]{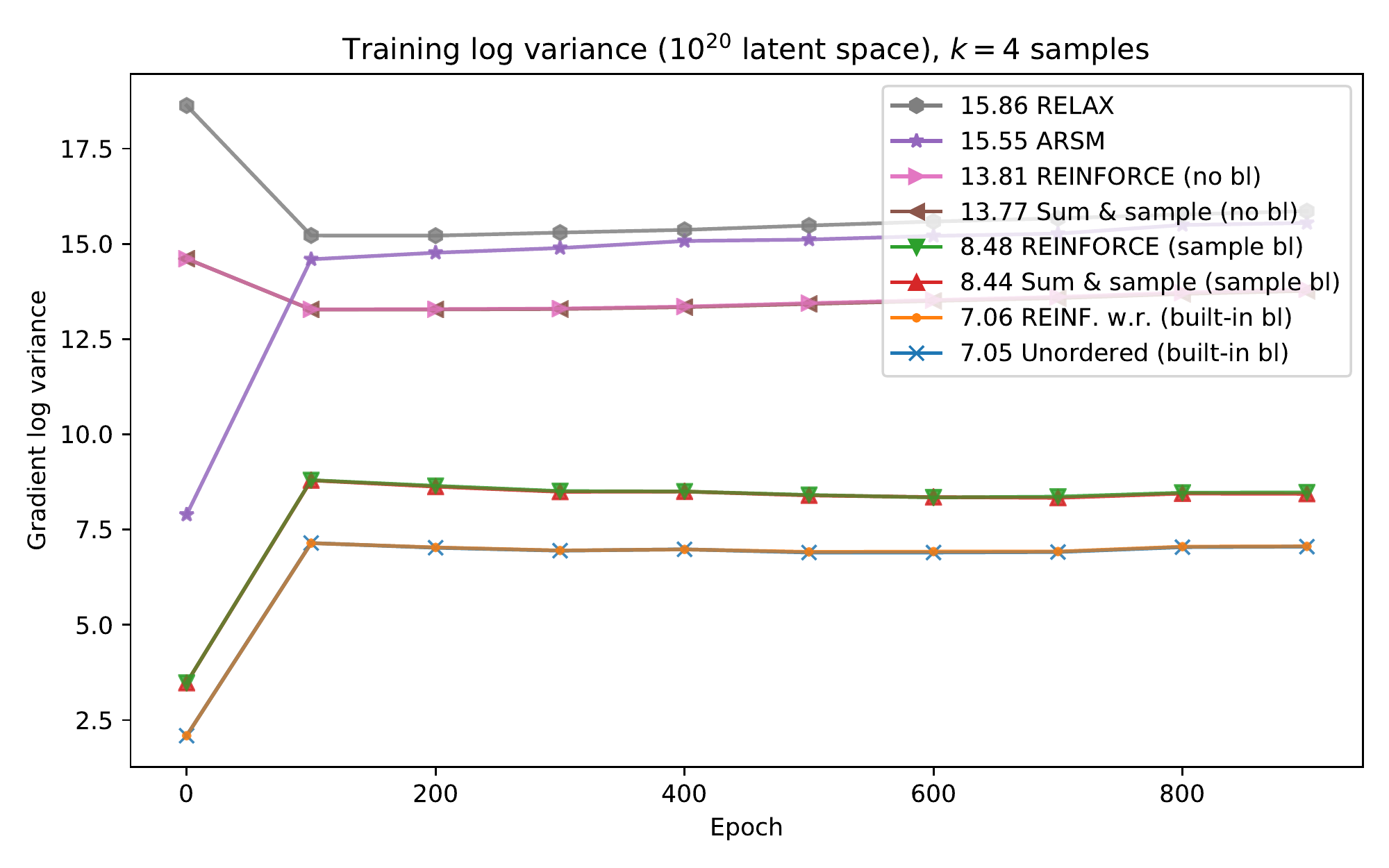}
        \vskip -0.2cm
        \caption{Large domain (latent space size $10^{20}$)}
    \end{subfigure}
        \caption{Gradient log variance of different unbiased estimators with $k = 4$ samples, estimated every 100 (out of 1000) epochs while training using REINFORCE with replacement. Each estimator is computed 1000 times with different latent samples for a fixed minibatch (the first 100 records of training data). We report (the logarithm of) the sum of the variances per parameter (trace of the covariance matrix). Some lines coincide, so we sort the legend by the last measurement and report its value. }
        \label{fig:vae_grads}
\end{figure}

\begin{figure}
    \centering
    \begin{subfigure}[b]{0.48\textwidth}
        \includegraphics[width=\textwidth]{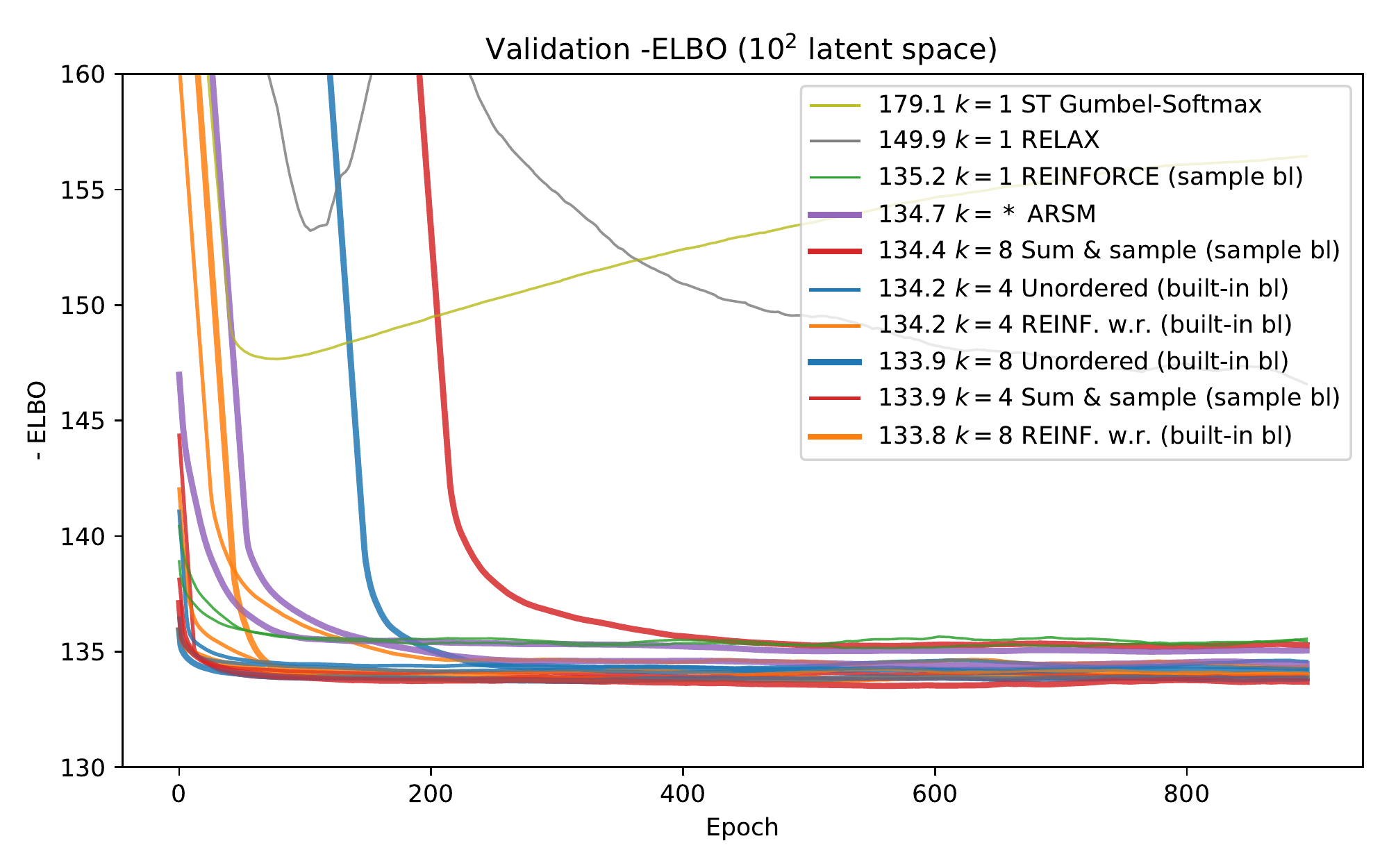}
        \caption{Small domain (latent space size $10^2$)}
    \end{subfigure}
    \begin{subfigure}[b]{0.48\textwidth}
        \includegraphics[width=\textwidth]{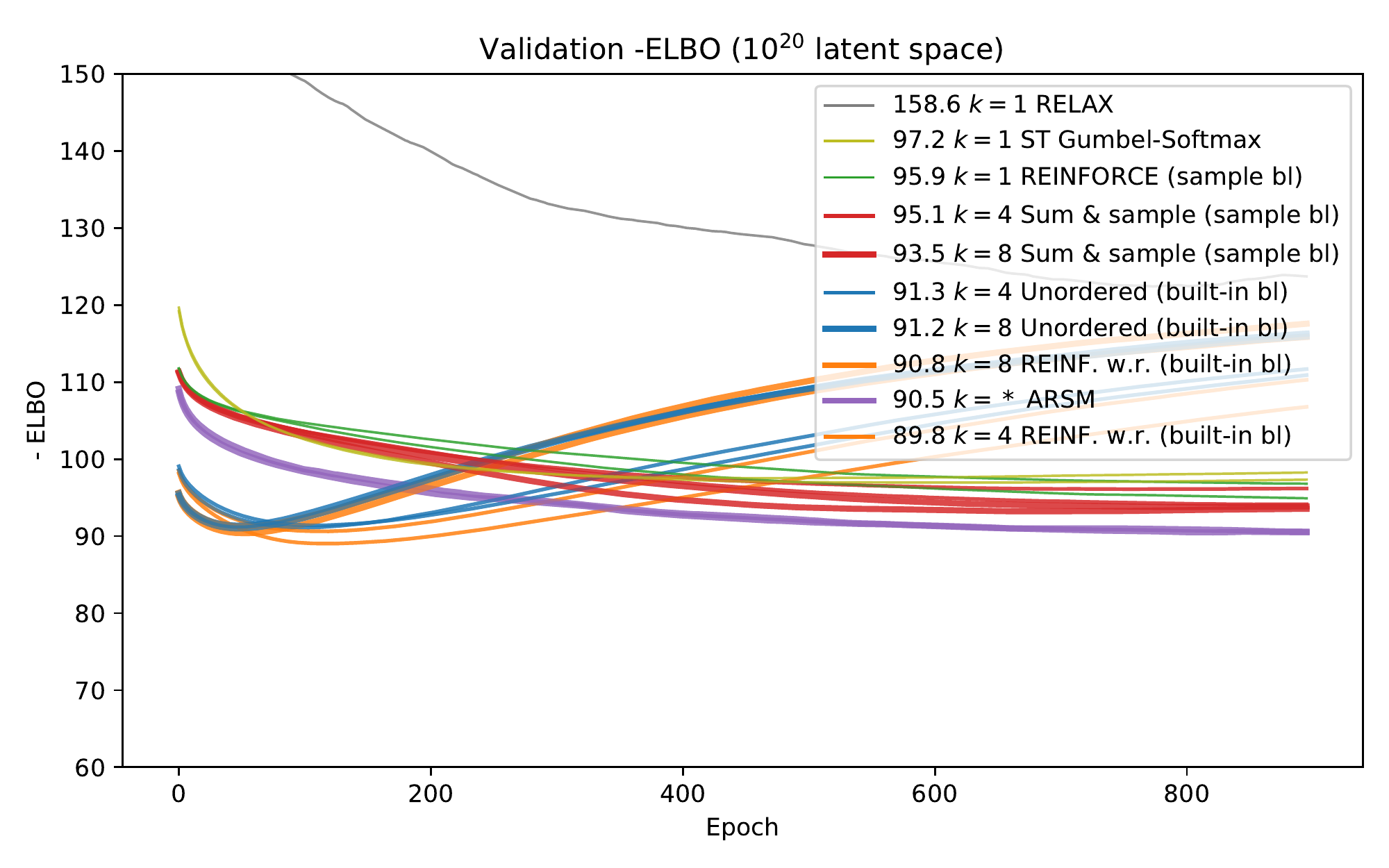}
        \caption{Large domain (latent space size $10^{20}$)}
    \end{subfigure}
        \caption{Smoothed validation -ELBO curves during training of two independent runs when with different estimators with $k =1$, 4 or 8 (thicker lines) samples (ARSM has a variable number). Some lines coincide, so we sort the legend by the lowest -ELBO achieved and report this value.}
        \label{fig:vae_valid_elbo}
\end{figure}
\paragraph{Negative ELBO on validation set.}
Figure \ref{fig:vae_valid_elbo} shows the -ELBO evaluated during training on the validation set. For the large latent space, we see validation error quickly increase (after reaching a minimum) which is likely because of overfitting (due to improved optimization), a phenomenon observed before \citep{tucker2017rebar, grathwohl2017backpropagation}. Note that before the overfitting starts, both REINFORCE without replacement and the unordered set estimator achieve a validation error similar to the other estimators, such that in a practical setting, one can use early stopping.

\paragraph{Results using standard binarized MNIST dataset.}
Instead of using the MNIST dataset binarized by thresholding values at 0.5 (as in the code and paper by \citet{yin2019arsm}) we also experiment with the standard (fixed) binarized dataset by \citet{salakhutdinov2008quantitative,larochelle2011neural}, for which we plot train and validation curves for two runs on the small and large domain in Figure \ref{fig:vae_standard_mnist}. This gives more realistic (higher) -ELBO scores, although we still observe the effect of overfitting. As this is a bit more unstable setting, one of the runs using REINFORCE with replacement diverged, but in general the relative performance of estimators is similar to using the dataset with 0.5 threshold.
\begin{figure}
    \centering
    \begin{subfigure}[b]{0.48\textwidth}
        \includegraphics[width=\textwidth]{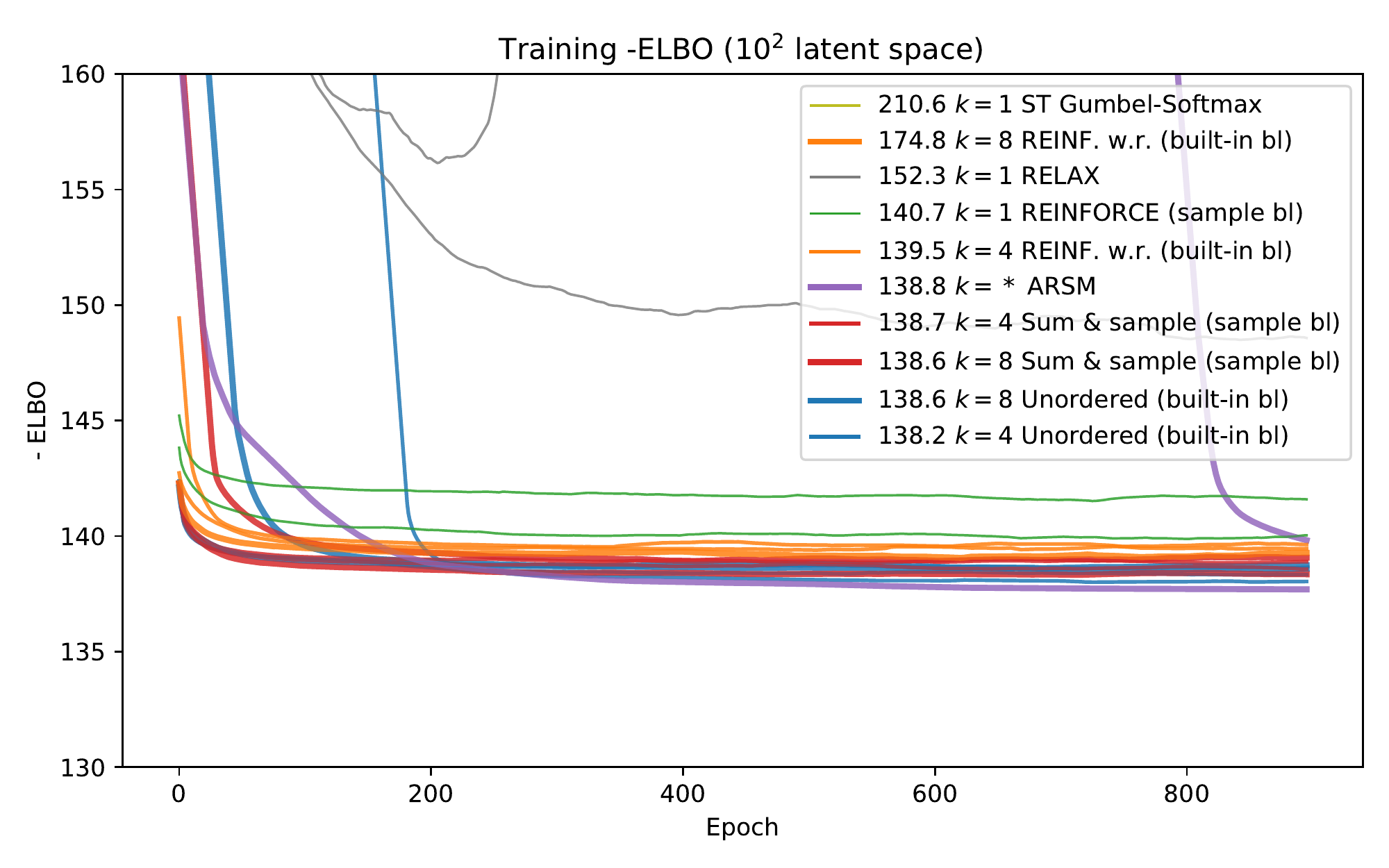}
        \caption{Training -ELBO, small domain ($10^2$)}
    \end{subfigure}
    \begin{subfigure}[b]{0.48\textwidth}
        \includegraphics[width=\textwidth]{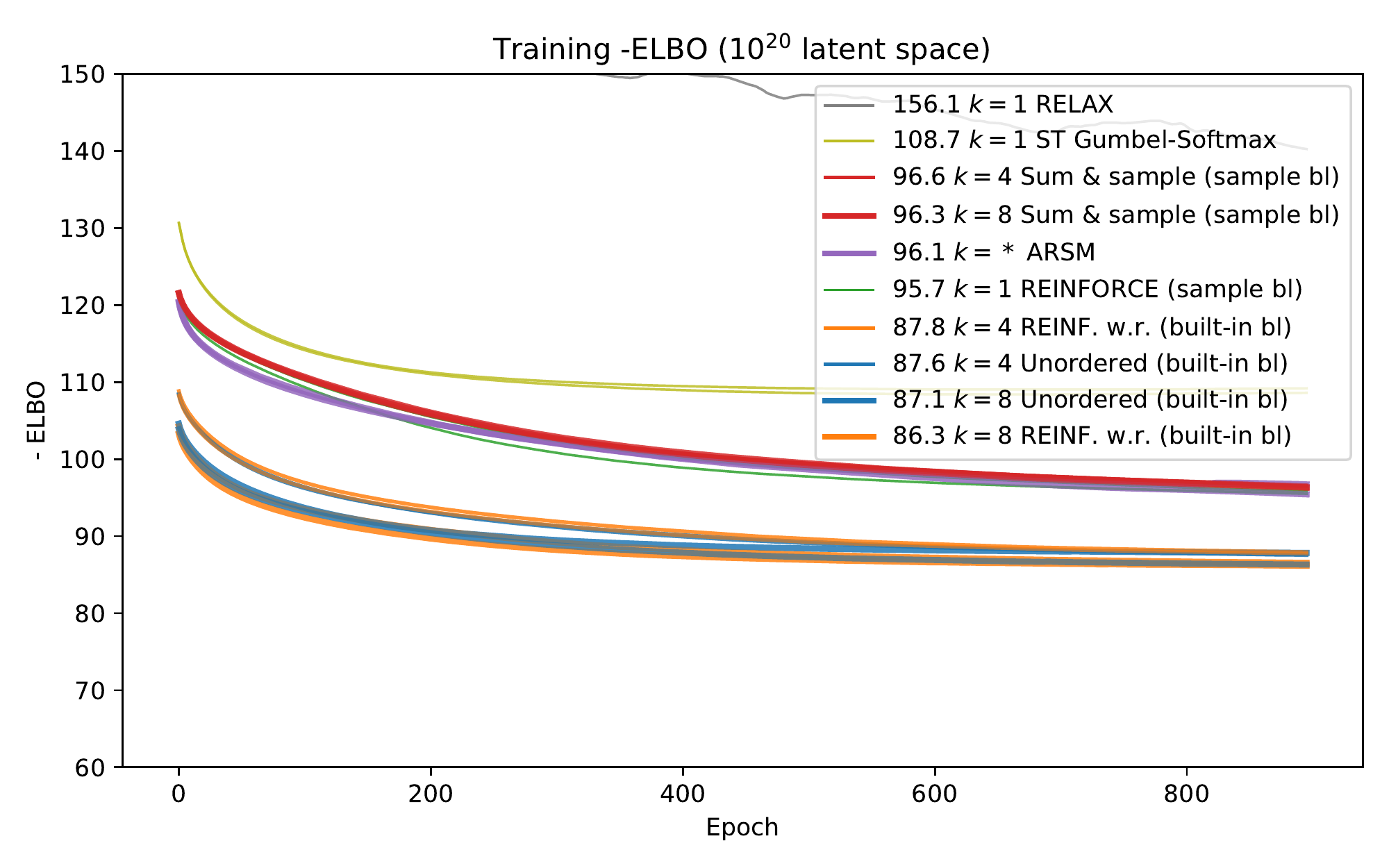}
        \caption{Training -ELBO, large domain ($10^{20}$)}
    \end{subfigure}
    
    \begin{subfigure}[b]{0.48\textwidth}
        \includegraphics[width=\textwidth]{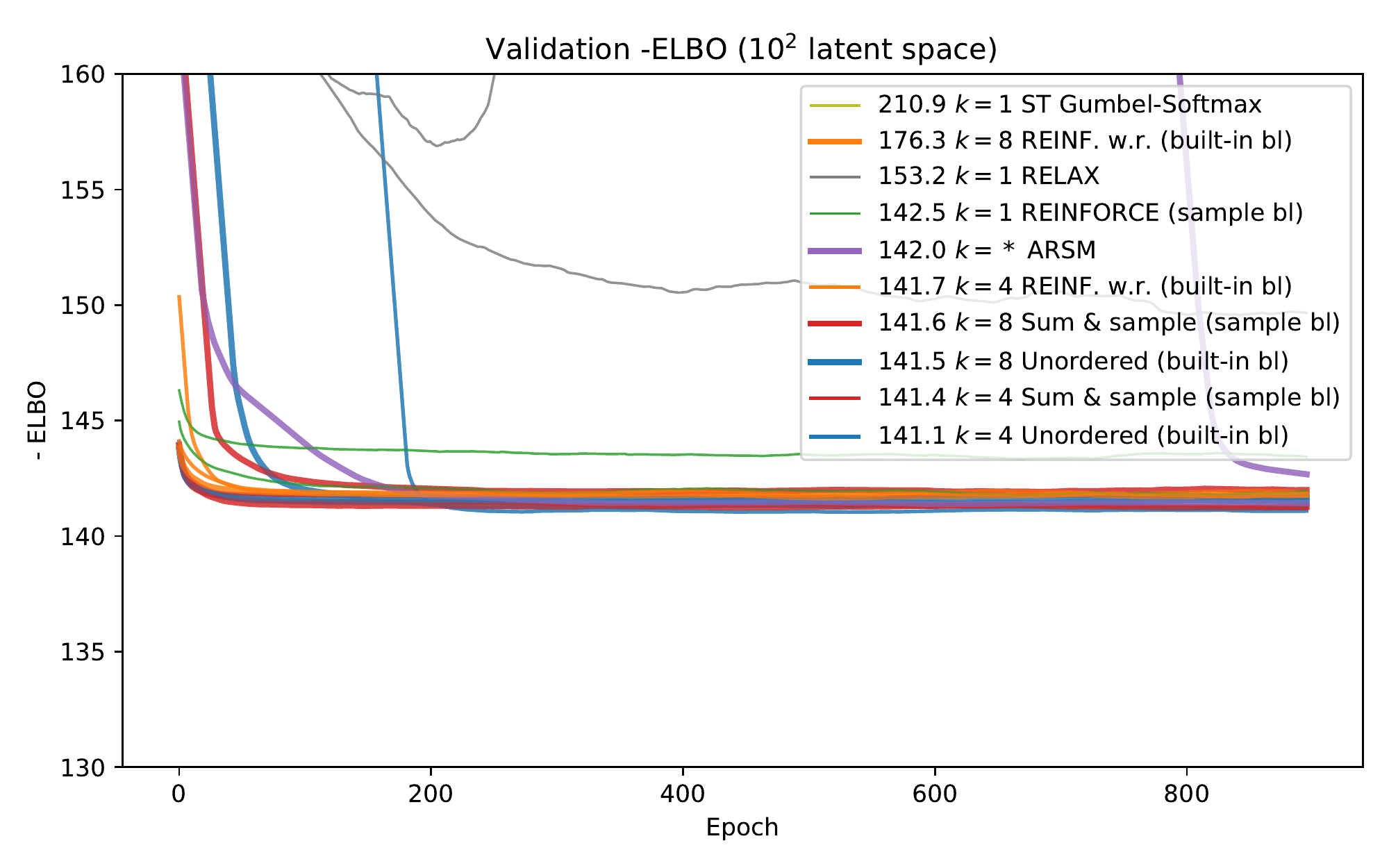}
        \caption{Validation -ELBO, small domain ($10^2$)}
    \end{subfigure}
    \begin{subfigure}[b]{0.48\textwidth}
        \includegraphics[width=\textwidth]{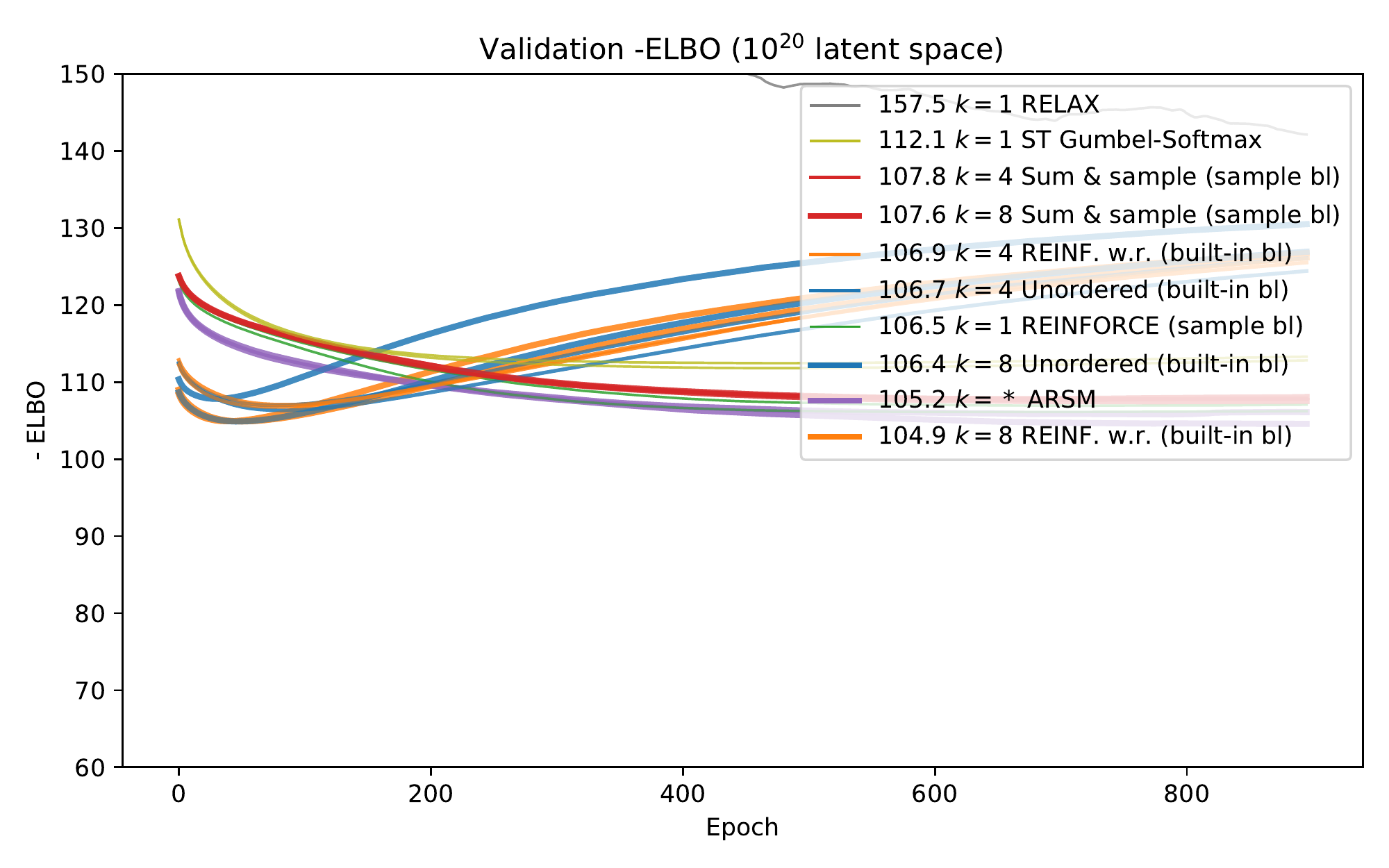}
        \caption{Validation -ELBO, large domain ($10^{20}$)}
    \end{subfigure}
        \caption{Smoothed training and validation -ELBO curves during training on the standard binarized MNIST dataset \citep{salakhutdinov2008quantitative, larochelle2011neural} of two independent runs when with different estimators with $k =1$, 4 or 8 (thicker lines) samples (ARSM has a variable number). Some lines coincide, so we sort the legend by the lowest -ELBO achieved and report this value.}
        \label{fig:vae_standard_mnist}
\end{figure}

\clearpage
\section{Travelling Salesman Problem}
\label{app:tsp_details}
The Travelling Salesman Problem (TSP) is a discrete optimization problem that consists of finding the order in which to visit a set of locations, given as $x,y$ coordinates, to minimize the total length of the tour, starting and ending at the same location. As a tour can be considered a sequence of locations, this problem can be set up as a sequence modelling problem, that can be either addressed using supervised \citep{vinyals2015pointer} or reinforcement learning \citep{bello2016neural,kool2018attention}.

\citet{kool2018attention} introduced the Attention Model, which is an encoder-decoder model which considers a TSP instances as a fully connected graph. The encoder computes embeddings for all nodes (locations) and the decoder produces a tour, which is sequence of nodes, selecting one note at the time using an attention mechanism, and uses this autoregressively as input to select the next node. In \citet{kool2018attention}, this model is trained using REINFORCE, with a greedy rollout used as baseline to reduce variance.

We use the code by \citet{kool2018attention} to train the exact same Attention Model (for details we refer to \citet{kool2018attention}), and minimize the expected length of a tour predicted by the model, using different gradient estimators. We did not do any hyperparameter optimization and used the exact same training details, using the Adam optimizer \citep{kingma2015adam} with a learning rate of $10^{-4}$ (no decay) for 100 epochs for all estimators. For the baselines, we used the same batch size of 512, but for estimators that use $k = 4$ samples, we used a batch size of $\frac{512}{4} = 128$ to compensate for the additional samples (this makes multi-sample methods actually faster since the encoder still needs to be evaluated only once).

\end{document}